\documentclass[12pt,a4paper]{article}
\usepackage{fullpage}
\usepackage[authoryear,round]{natbib}
\usepackage{authblk}

\usepackage{graphicx}
\usepackage{epstopdf}
\usepackage{subfigure} 
\usepackage{amsmath}
\usepackage{amsfonts}
\usepackage{amssymb}
\usepackage{amsthm}
\usepackage{mathrsfs}
\usepackage{algorithm}
\usepackage{algorithmic}
\usepackage{caption}
\usepackage{dsfont}

\newtheorem{theorem}{Theorem}

\newtheorem{lemma}[theorem]{Lemma}

\theoremstyle{remark}
\newtheorem{remark}{Remark}
\newtheorem*{remark*}{Remark}

\DeclareMathOperator{\sign}{\mathrm{sign}}

\DeclareMathOperator*{\argmin}{\arg\min}
\DeclareMathOperator*{\argmax}{\arg\max}

\newcommand{\bE}{\mathbb{E}}

\newcommand{\ve}{\boldsymbol{e}}

\newcommand{\vZero}{\boldsymbol{0}}
\newcommand{\cA}{\mathcal{A}}

\newcommand{\cI}{\mathcal{I}}
\newcommand{\cS}{\mathcal{S}}

\newcommand{\cT}{\mathcal{T}}
\newcommand{\cC}{\mathcal{C}}

\newcommand{\dOne}{\mathds{1}}

\title{Bandit-Based Task Assignment for\\
Heterogeneous Crowdsourcing}
\author[1]{Hao Zhang}
\author[1]{Yao Ma}
\author[2]{Masashi Sugiyama}
\affil[1]{\normalsize Department of Computer Science, Tokyo Institute of Technology, Japan}
\affil[2]{\normalsize Department of Complexity Science and Engineering, The University of Tokyo, Japan}

\begin{document}

\maketitle

\begin{abstract}
We consider a task assignment problem in crowdsourcing, which is aimed at collecting as many reliable labels as possible within a limited budget. A challenge in this scenario is how to cope with the diversity of tasks and the task-dependent reliability of workers, e.g., a worker may be good at recognizing the name of sports teams, but not be familiar with cosmetics brands. We refer to this practical setting as \emph{heterogeneous} crowdsourcing. In this paper, we propose a contextual bandit formulation for task assignment in heterogeneous crowdsourcing, which is able to deal with the exploration-exploitation trade-off in worker selection. We also theoretically investigate the regret bounds for the proposed method, and demonstrate its practical usefulness experimentally.
\end{abstract}

\section{Introduction}
\label{sec:intro}

The quantity and quality of labeled data significantly affect the performance of machine learning algorithms. However, collecting reliable labels from experts is usually expensive and time-consuming. The recent emergence of \emph{crowdsourcing} services such as \emph{Amazon Mechanical Turk}\footnote{https://www.mturk.com/mturk/} (MTurk) enables us to cheaply collect huge quantities of labeled data from crowds of workers for many machine learning tasks, e.g., natural language processing \citep{snow08} and computer vision \citep{welinder10a}. In a crowdsourcing system, a requester asks workers to complete a set of labeling tasks by paying them a tiny amount of money for each label.

The primary interest of crowdsourcing research has been how to cope with different reliability of workers and aggregate the collected noisy labels \citep{dawid79,smyth94,ipeirotis10,raykar10,welinder10b,yan10,kajino12,liu12,zhou12}. Usually, \emph{weighted voting} mechanism is implicitly or explicitly utilized for label aggregation, with workers' reliability as weights. Many existing methods use \emph{Expectation-Maximization} (EM) \citep{dempster77} on \emph{static} datasets of the collected labels to jointly estimate  workers' reliability and true labels. However, how to \emph{adaptively} collect these labels is often neglected. Since the total budget for a requester to pay the workers is usually limited, it is necessary to consider how to intelligently use the budget to assign tasks to workers.

This leads another important line of crowdsourcing research, which is called the task routing or task assignment problem. There are two classes of task assignment methods: \emph{push} and \emph{pull}. While in \emph{pull} methods the system takes a passive role and only sets up the environment for workers to find tasks themselves, in \emph{push} methods the system takes complete control over which tasks are assigned to whom \citep{law11}. In this paper we focus on \emph{push} methods, and refer to them as task assignment methods henceforth. Most of the existing task assignment methods run in an online mode, simultaneously learning workers' reliability and collecting labels \citep{donmez09,chen13,ertekin14}. To deal with the exploration (i.e. learning which workers are reliable) and exploitation (i.e. selecting the workers considered to be reliable) trade-off in worker selection, \emph{IEThresh} \citep{donmez09} and \emph{CrowdSense} \citep{ertekin14} dynamically sample worker subsets according to workers' labeling performances. However, this is not enough in recent heterogeneous crowdsourcing where a worker may be reliable at only a subset of tasks with a certain type. For example, the workers considered to be good at previous tasks in IEThresh and CrowdSense may be bad at next ones. Therefore, it is more reasonable to model task-dependent reliability for workers in heterogeneous crowdsourcing. Another issue in IEThresh and CrowdSense is that the budget is not pre-fixed. That is, the requester will not know the total budget until the whole task assignment process ends. This makes those two methods not so practical for crowdsourcing. OptKG \citep{chen13} runs within a pre-fixed budget and formulates the task assignment problem as a Markov decision process (MDP). However, it is difficult to give theoretical guarantees for OptKG when heterogeneous workers are involved.

In recent crowdsourcing markets, as the heterogeneity of tasks is increasing, many researchers started to focus on heterogeneous crowdsourcing. \citet{goel14} studied the problem of incentive-compatible mechanism design for heterogeneous markets. The goal is to properly price the tasks for worker trustfulness and maximize the requester utility with the financial constrain. \citet{ho12} and \citet{ho13} studied the problem of task assignment in heterogeneous crowdsourcing. However, it is another variant of problem setting, where workers arrive online and the requester must assign a task (or sequence of tasks) to each new worker as she arrives \citep{slivkins14}. While in our problem setting, the requester completely controls which task to pick and which worker to select at each step.

From a technical perspective, the most similar problem setting to ours is that of OptKG, where we can determine a task-worker pair at each step. For the purpose of extensive comparison, we also include two heuristic methods IEThresh and CrowdSense as well as OptKG in our experiments. These three task assignment methods are further detailed in Section \ref{sec:exist}.

In this paper, we propose a contextual bandit formulation for task assignment in heterogeneous crowdsourcing. Our method models task-dependent reliability for workers by using \emph{weight}, which depends on the \emph{context} of a certain task. Here, \emph{context} can be interpreted as the type or required skill of a task. For label aggregation, we adopt \emph{weighted voting}, which is a common solution used for aggregating noisy labels in crowdsourcing. Our method consists of two phases: the \emph{pure exploration} phase and the \emph{adaptive assignment} phase. In the pure exploration phase, we explore workers' reliability in a batch mode and initialize their weights as the input of the adaptive assignment phase. On the other hand, the adaptive assignment phase includes a bandit-based strategy, where we sequentially select a worker for a given labeling task with the help of the \emph{exponential weighting scheme} \citep{cesa-bianchi06,arora12}, which is a standard tool for bandit problems. The whole method runs within a limited budget. Moreover, we also investigate the regret bounds of our strategy theoretically, and demonstrate its practical usefulness experimentally.

The rest of this paper is organized as follows. In Section \ref{sec:method}, we describe our proposed bandit-based task assignment method for heterogeneous crowdsourcing. Then we theoretically investigate its regret bounds in Section \ref{sec:theory}. For the purpose of comparison, we look into the details of the existing task assignment methods for crowdsourcing in Section \ref{sec:exist} and experimentally evaluate them together with the proposed method in Section \ref{sec:experiment}. Finally, we present our conclusions in Section \ref{sec:conclusion}.

\section{Bandit-Based Task Assignment}
\label{sec:method}

In this section, we describe our proposed bandit-based task assignment (BBTA) method.

\subsection{Problem Formulation}
Suppose we have $N$ unlabeled tasks with indices $\cI=\{1,\ldots,N\}$, each of which is characterized by a context $s$ from a given context set $\cS$, where $|\cS|=S$. Let $\{y_i^*\}_{i=1}^{N}$ be the unknown true labels of tasks, where $y_i^*\in \{-1,1\}$. Each time given a task, we ask one worker from a pool of $K$ workers for a (possibly noisy) label, consuming one unit of the total budget $T$. Our goal is to find suitable task-worker assignment to collect as many reliable labels as possible within the limited budget $T$. Finally, we aggregate the collected labels to estimate the true labels $\{y_i^*\}_{i=1}^{N}$.

Let $y_{i,j}$ be the individual label of task $i$ (with context $s_i$) given by worker $j$. If the label is missing, we set $y_{i,j}=0$. For simplicity, we omit the subscript $i$ of context $s_i$, and refer to the context of the current task as $s$. We denote the weight of worker $j$ for context $s$ by $w^s_{j}$ $(>0)$, corresponding to the task-dependent reliability of worker $j$. Note that $w^s_{j}$ is what we learn dynamically in the method. Then an estimate of the true label $y_i^*$ is calculated by the weighted voting mechanism as 
\begin{align}
\label{eq:weighted}
\widehat{y}_i=\sign(\overline{y}_i),\text{ where }\overline{y}_i=\frac{\sum_{j=1}^K w^s_{j}y_{i,j}}{\sum_{j'=1}^K w^s_{j'}}.
\end{align}

Our proposed method consists of the \emph{pure exploration} phase and the \emph{adaptive assignment} phase. The pseudo code is given in Algorithm \ref{alg:BBTA}, and the details are explained below.

\begin{algorithm}
\caption{Bandit-Based Task Assignment (BBTA)}
\label{alg:BBTA}
\begin{algorithmic}[1]
\STATE {\bfseries Input:} 
\STATE ~~~$N$: The number of tasks, each of which has a context $s$ from $\cS$
\STATE ~~~$K$: The number of workers 
\STATE ~~~$T$: The total budget 
\STATE ~~~$N'$: The number of tasks for each context in the pure exploration phase
\STATE {\bfseries Initialization:} 
\STATE ~~~$y_{i,j}=0$ for $i=1,\ldots,N$ and $j=1,\ldots,K$.
\STATE ~~~$\left|\overline{y}_{i}\right|=0$ for $i=1,\ldots,N$.
\STATE ~~~$t^s=0$ for $s\in\cS$.
\STATE \% \textsf{Pure Exploration Phase}
\STATE Pick $N^{\prime}$ tasks for each of $S$ distinct contexts.
\STATE Collect labels from \emph{all} workers for these tasks and use Eqn.~\ref{eq:majority} for label aggregation.
\STATE Calculate cumulative losses for $j=1,\ldots,K$ and $s\in\cS$:
\[
L^s_{j,0}=\sum_{i\in\cI^s_1}\dOne_{y_{i,j}\neq \widehat{y}_i}.
\]
\STATE Set the budget as $T_2=T-SKN'$ and the index set of available tasks as $\cI_2=\cI\setminus\bigcup_{s\in\cS}\cI_1^s$ for the adaptive assignment phase.
\STATE \% \textsf{Adaptive Assignment Phase}
\FOR{$t=1$ {\bfseries to} $T_2$}
\STATE Pick the task with the lowest confidence score
$\displaystyle
i_t=\argmin_{i\in \cI_2} \left|\overline{y}_i\right|$.
\STATE Observe its context $s$, and update $t^{s}\leftarrow t^{s}+1$.
\STATE Set $\eta^{s}_{t^{s}}=\sqrt{\frac{\ln K}{t^{s}K}}$, and calculate the weights for $j=1,\ldots,K$:
\[
w^s_{j,t^s}=\exp(-\eta^s_{t^s}L^s_{j,t^s-1}).
\]
\STATE Draw the worker $j_t$ from the distribution $(p_{1,t},\ldots,p_{K,t})$, where 
\[
p_{j,t}=\frac{w^s_{j,t^s}}{\sum_{j'=1}^{K}w^s_{j',t^s}}.
\]
\STATE Obtain label $y_{i_t,j_t}$ and calculate $\widehat{y}_{i_t}$ by the weighted voting mechanism (Eqn.~\ref{eq:weighted}).
\STATE Set 
$l_{j_t,t}=
\dOne_{\widehat{y}_{i_t}\neq y_{i_t,j_t}}$, and for $j=1,\ldots,K$ calculate $\widetilde{l}_{j,t}=\frac{l_{j,t}}{p_{j,t}}\dOne_{j=j_t}$ and update $L^s_{j,t^s}=L^s_{j,t^s-1}+\widetilde{l}_{j,t}$.
\STATE Update confidence scores $\left|\overline{y}_i\right|$ of all tasks with context $s$ by using Eqn.~\ref{eq:weighted}.
\IF{task $i_t$ is already labeled by all workers}
\STATE $\cI_2\leftarrow\cI_2\setminus\{i_t\}$.
\ENDIF
\ENDFOR
\STATE {\bfseries Output: }$\widehat{y}_i=\sign(\overline{y}_i)$ for $i=1,\ldots,N$
\end{algorithmic}
\end{algorithm}

\subsection{Pure Exploration Phase}
Pure exploration performs in a batch mode, and the purpose is to know which workers are reliable at which labeling tasks. To this end, we pick $N'$ tasks for each of $S$ distinct contexts ($SN' \ll N$) and let \emph{all} $K$ workers label them (Line 11-12 in Algorithm \ref{alg:BBTA}). We denote the index set of $N'$ tasks with context $s$ as $\cI^s_1$ in this phase.

Since we have no prior knowledge of workers' reliability in this phase, we give all of them the same weight when aggregating their labels (equivalent to majority voting):
\begin{align}
\label{eq:majority}
\widehat{y}_i=\sign(\overline{y}_i),\text{ where }\overline{y}_i=\frac{1}{K}\sum_{j=1}^K y_{i,j}.
\end{align}
In the standard crowdsourcing scenario, all of the true labels are usually unknown. As in many other crowdsourcing methods, we have the prior belief that most workers perform reasonably well. To evaluate an individual label $y_{i,j}$, using the weighted vote $\widehat{y}_i$ is a common solution \citep{donmez09,ertekin14}. We denote the cumulative loss by $L^s_{j,0}$ and initialize it for the next phase as
\begin{align*}
L^s_{j,0}=\sum_{i\in\cI^s_1}\dOne_{y_{i,j}\neq \widehat{y}_i},\text{ for }j=1,\ldots,K\text{ and }s\in\cS,
\end{align*}
where $\dOne_{\pi}$ denotes the indicator function that outputs 1 if condition $\pi$ holds and 0 otherwise. This means that when a worker gives an individual label $y_{i,j}$ inconsistent or consistent with the weighted vote $\widehat{y}_i$, this worker suffers a loss 1 or 0. It is easy to see that cumulative losses correspond to workers' reliability. They are used for calculating workers' weights in the next phase. The budget for the next phase is $T_2=T-T_1$, where $T_1=SKN'<T$ is the budget consumed in this phase.

\subsection{Adaptive Assignment Phase}
In the adaptive assignment phase, task-worker assignment is determined for the remaining $N-SN'$ tasks in an online mode within the remaining budget $T_2$. At each step $t$ of this phase, to determine a task-worker pair, we need to further consider which task to pick and which worker to select for this task.

According to the weighted voting mechanism (Eqn.~\ref{eq:weighted}), the magnitude $|\overline{y}_i|$ $(\in[0,1])$ corresponds to the \emph{confidence score} of $\widehat{y}_i$. For task $i$, the confidence score $|\overline{y}_i|$ will be $1$ if and only if we have collected labels from \emph{all} workers and \emph{all} of them are consistent. On the other hand, when the sum of (normalized) weights for positive labels is equal to that for negative labels, or we have no labels for task $i$, the confidence score $|\overline{y}_i|$ is $0$. That is, we are not confident in the aggregated label (i.e. the confidence score is low) because the collected labels are significantly inconsistent, or insufficient, or both. If the confidence score of a task is lower than those of others, collecting more labels for this task is a reasonable solution. Thus we pick task $i_t$ with the lowest confidence score as the next one to label (Line 17 in Algorithm \ref{alg:BBTA}):
\begin{align*}
i_t=\argmin_{i\in\cI_2} \left|\overline{y}_i\right|,
\end{align*}
where $\cI_2$ is the index set of current available tasks in this phase. This idea is similar to \emph{uncertainty sampling} \citep{lewis94} in active learning.

Given the picked task, selecting a worker reliable at this task is always favored. On the other hand, workers' reliability is what we are dynamically learning in the method. There exists a trade-off between exploration (i.e. learning which worker is reliable) and exploitation (i.e. selecting the worker considered to be reliable) in worker selection. To address this trade-off, we formulate our task assignment problem as a \emph{multi-armed bandit} problem, more specifically, a \emph{contextual bandit} problem \citep{bubeck12}.

Multi-armed bandit problems \citep{auer02a,auer02b} are basic examples of sequential decision making with \emph{limited feedback} and can naturally handle the exploration-exploitation trade-off. At each step, we allocate one unit of budget to one of a set of actions and obtain some observable reward (loss) given by the environment. The goal is to maximize (minimize) the cumulative reward (loss) in the whole allocation sequence. To achieve this goal, we must balance the exploitation of actions that are good in the past and the exploration of actions that may be better in the future. A practical extension of the basic bandit problem is the contextual bandit problem, where each step is marked by a context from a given set. Then the interest is finding good mappings from contexts to actions rather than identifying good actions.

In most contextual bandits, the context (side information) is provided as a feature vector which influences the reward/loss at each step, whereas in the heterogeneous crowdsourcing setting, we consider the simple situation of contextual bandits, where the context is marked by some discrete value, corresponding to the task type. Then our interest is to find a good mapping from task types to appropriate workers rather than the best worker for all tasks. Note that the task types are observable and provided by the environment.

The adaptive assignment phase includes our strategy of worker selection, where selecting a worker corresponds to taking an action. The objective is to collect reliable labels via the strategy of adaptively selecting workers from the worker pool.

We further assume the scheme that each worker adopts to assign labels as a black box. Then the labeled status of all tasks could abruptly change over time. This means the exact sequence of tasks in the adaptive assignment phase is unpredictable, given that we calculate confidence scores for all tasks based on their labeled status. Thus we consider the task sequence as the external information in this phase.

Given the task with context $s$, we calculate the weights ($s$-dependent) of workers as follows (Line 19 in Algorithm \ref{alg:BBTA}):
\begin{align*}
w^s_{j,t^s}=\exp(-\eta^s_{t^s}L^s_{j,t^s-1}),\text{ for }j=1,\ldots,K,
\end{align*}
where $t^s$ is the appearance count of context $s$ and $\eta^s_{t^s}$ is the learning rate related to $t^s$. This calculation of weights by using cumulative losses is due to the \emph{exponential weighting scheme} \citep{cesa-bianchi06,arora12}, which is a standard tool for sequential decision making under adversarial assumptions. Following the exponential weighting scheme, we then select a worker $j_t$ from the discrete probability distribution on workers with each probability $p_{j,t}$ proportional to the weight $w^s_{j,t^s}$ (Line 20 in Algorithm \ref{alg:BBTA}). Since workers' cumulative losses do not greatly vary at the beginning but gradually differ from each other through the whole adaptive assignment phase, this randomized strategy can balance the exploration-exploitation trade-off in worker selection, by exploring more workers at earlier steps and doing more exploitation at later steps.

Then we ask worker $j_t$ for an individual label $y_{i_t,j_t}$ and calculate the weighted vote $\widehat{y}_{i_t}$ by using Eqn.~\ref{eq:weighted} (Line 21 in Algorithm \ref{alg:BBTA}). With the weighted vote $\widehat{y}_{i_t}$, we obtain the loss of the selected worker $j_t$: $l_{j_t,t}=\dOne_{\widehat{y}_{i_t}\neq y_{i_t,j_t}}$ (Line 22 in Algorithm \ref{alg:BBTA}). Recall that we do not have any assumption on how each worker gives a label. Then we have no stochastic assumption on the generation of losses. Thus we can consider this problem as an adversarial bandit problem. Note that we can only observe the loss of the selected worker $j_t$, and for other workers, the losses are unobserved. This is called \emph{limited feedback} in bandit problems. Here, we decided to give an unbiased estimate of loss for any worker from the above distribution: 
\begin{align*}
\widetilde{l}_{j,t}=\frac{l_{j,t}}{p_{j,t}}\dOne_{j=j_t}.
\end{align*}
It is easy to see that the expectation of $\widetilde{l}_{j,t}$ with respect to the selection of worker $j_t$ is exactly the loss $l_{j,t}$. Finally, we update the cumulative losses and the confidence scores of tasks with the same context as the current one (Line 22-23 in Algorithm \ref{alg:BBTA}).

The above assignment step is repeated $T_2$ times until the budget is used up.

\section{Theoretical Analysis}
\label{sec:theory}

In this section, we theoretically analyze the proposed bandit-based task assignment (BBTA) method.

The behavior of a bandit strategy is studied by means of \emph{regret} analysis. Usually, the performance of a bandit strategy is compared with that of the optimal one, to show the ``regret'' for not following the optimal strategy. In our task assignment problem, we use the notion of regret to investigate how well the proposed strategy can select better workers from the whole worker pool, by comparing our strategy with the optimal one.

The reason why we use regret as the evaluation measure is that the whole task assignment process is working in an online mode. From the perspective of a requester, the objective is to maximize the average accuracy of the estimated true labels with the constraint of the budget. Ideally, this is possible when we have complete knowledge of the whole process, e.g. the reliability of each worker and the budget for each context. However, in the setting of task assignment, since we can not know beforehand any information about the worker behaviors and the coming contexts in the future, it is not meaningful to try to maximize the average accuracy. Instead, the notion of regret can be used as an evaluation measure for the strategy of worker selection in the proposed method, which evaluates a relative performance loss compared with the optimal strategy. As a general objective for online problems \citep{hazan11, shwartz12}, minimizing the regret is a common and reasonable approach to guaranteeing the performance.

Specifically, we define the regret by
\begin{align*}
\overline{R}_{T}=\max_{g: \cS\rightarrow\{1,\ldots,K\}}\bE\left[\sum_{t=1}^{T_2}(l_{j_t,t}-l_{g(s),t})\right],
\end{align*}
where $g: \cS\rightarrow\{1,\ldots,K\}$ is the mapping from contexts to workers, $s$ is the context of the task at step $t$, and $T_2=T-SKN'$ is the budget for the adaptive assignment phase. The regret we defined is the difference of losses between our strategy and the expected optimal one. Note that the optimal strategy here is assumed to be obtained in hindsight
under the same sequence of contexts as that of our strategy. The expectation is taken with respect to both the randomness of loss generation and worker selection $j_{t}\sim p_{t}$, where $p_{t}$ is the current worker distribution at step $t$ of our strategy. Since we consider the loss generation as \emph{oblivious} (i.e. the loss $l_{j_t,t}$ only depends on the current step $t$), we do not take the expectation over all sample paths.

We further denote $\cT_2^s$ as the set of indices for steps with context $s$, and $T_2^s=\left|\cT_2^s\right|$ as the total appearance count of context $s$ in the adaptive assignment phase. The following theorem shows the regret bound of BBTA.
\begin{theorem}
\label{thm}
The regret of BBTA with $\eta^{s}_{t^{s}}=\sqrt{\frac{\ln K}{t^{s}K}}$ is bounded as
\begin{align*}
\overline{R}_T\leq 2\sqrt{(T-SKN')SK\ln K}+\frac{S{N'}^2}{4}\sqrt{\frac{\ln K}{K}}+SN'.
\end{align*}
\end{theorem}
\begin{proof}
\begin{align*}
\overline{R}_{T}&=\max_{g: \cS\rightarrow\{1,\ldots,K\}}\bE\left[\sum_{t=1}^{T_2}(l_{j_t,t}-l_{g(s),t})\right]\\
&=\sum_{s\in\cS}\max_{j=1,\ldots,K}\bE\left[\sum_{t\in\cT_2^s}(l_{j_t,t}-l_{j,t})\right]\\
&\leq\sum_{s\in\cS}\left(2\sqrt{T_2^s K\ln K}+\frac{{N'}^2}{4}\sqrt{\frac{\ln K}{K}}+N'\right)\\
&\leq 2\sqrt{(T-SKN')SK\ln K}+\frac{S{N'}^2}{4}\sqrt{\frac{\ln K}{K}}+SN',
\end{align*}
where the first inequality is implied by Lemma~\ref{lm} shown below and the second one is due to Jensen's inequality and $\sum_{s\in\cS}T_2^s=T_2$.
\end{proof}

Now, let us consider the following setting. Suppose there is only one context ($S=1$) for all tasks (i.e. in the homogeneous setting). Then we do not have to distinguish different contexts. In particular, $t^s$, $\eta^{s}_{t^{s}}$, $L^s_{j,t^s}$, $w^s_{j,t^s}$ and $T_2^s$ are respectively equivalent to $t$, $\eta_{t}$, $L_{j,t}$, $w_{j,t}$ and $T_2$ in this setting. For convenience, we omit the superscript $s$ in this setting. The regret now is 
\begin{align*}
\overline{R}'_{T}=\max_{j\in\{1,\ldots,K\}}\bE\left[\sum_{t=1}^{T_2}(l_{j_t,t}-l_{j,t})\right],
\end{align*}
which is the difference of losses between our strategy and the one that always selects the best worker in expectation. The following Lemma~\ref{lm} shows a bound of $\overline{R}'_{T}$.
\begin{lemma}
\label{lm}
If there is only one context, then the regret of BBTA with $\eta_{t}=\sqrt{\frac{\ln K}{tK}}$ is bounded as
\begin{align*}
\overline{R}'_T\leq 2\sqrt{(T-KN')K\ln K}+\frac{{N'}^2}{4}\sqrt{\frac{\ln K}{K}}+N'.
\end{align*}
\end{lemma}

The proof of Lemma~\ref{lm} is provided in Appendix.

\begin{remark}
Theorem~\ref{thm} and Lemma~\ref{lm} show sub-linear regret bounds with respect to $T$ for BBTA, indicating that the performance of our strategy converges to that of the expected optimal one as the budget $T$ goes to infinity.
\end{remark}

\begin{remark}
The pure exploration phase is necessary for obtaining some knowledge of workers' reliability before adaptively assigning tasks, and $N'$ controls the length of the pure exploration phase. In Theorem~\ref{thm}, it is easy to see that a larger $N'$ makes the term $2\sqrt{(T-SKN')SK\ln K}$ (corresponding to the adaptive assignment phase) smaller but results in larger $\frac{{N'}^2}{4}\sqrt{\frac{\ln K}{K}}+N'$ (related to the pure exploration phase). Moreover, a longer pure exploration phase also consumes a larger proportion of the total budget. The pure exploration phase is effective (as we will see in Section~\ref{sec:experiment}), but how effective it can be is not only related to the choice of $N'$ (even we could obtain the optimal $N'$ that minimizes the bound), but also depends on many other factors, such as the true accuracy of workers and how different their labeling performances are from each others, which are unfortunately unknown in advance. A practical choice used in our experiments in Sec.~\ref{sec:experiment} is to set a very small $N'$ (e.g. $N'=1$).
\end{remark}

\section{Review of Existing Task Assignment Methods for Crowdsourcing}
\label{sec:exist}

In this section, we review the existing task assignment methods for crowdsourcing and discuss similarities to and differences from the proposed method.

\subsection{IEThresh}

IEThresh \citep{donmez09} builds upon \emph{interval estimation} (IE) \citep{kaelbling90}, which is used in reinforcement learning for handling the exploration-exploitation trade-off in action selection. Before each selection, the IE method estimates the \emph{upper confidence interval} (UI) for each action $a$ according to its cumulative reward:
\begin{align*}
\mathrm{UI}(a)=m(a)+t_{\frac{\alpha}{2}}^{(n-1)}\frac{s(a)}{\sqrt{n}},
\end{align*}
where $m(a)$ and $s(a)$ are the mean and standard deviation of rewards that has been received so far by choosing action $a$, $n$ denotes the number of times $a$ has been selected so far, and $t_{\frac{\alpha}{2}}^{(n-1)}$ is the critical value for Student's $t$-distribution with $n-1$ degrees of freedom at significance level $\alpha/2$. In experiments, we set $\alpha$ at 0.05, following \citet{donmez09}.

In IEThresh, taking the action $a_j$ corresponds to selecting worker $j$ to ask for a label. The reward is 1 if the worker's label agrees with the majority vote, and 0 otherwise. At each step, a new task arrives, UI scores of all workers are re-calculated, and a subset of workers with indices
\begin{align*}
\{j\mid \mathrm{UI}(a_j)\geq\epsilon\cdot\max_a \mathrm{UI}(a)\}
\end{align*}
are asked for labels. It is easy to see that the pre-determined threshold parameter $\epsilon$ controls the size of worker subset at each step.

\subsection{CrowdSense}

Similarly to IEThresh, CrowdSense \citep{ertekin14} also dynamically samples worker subsets in an online mode. The criterion that CrowdSense uses for handling exploration-exploitation in worker selection is described as follows. 

Before assigning a new task $i$ to workers, CrowdSense first calculates a quality estimate for each worker:
\begin{align*}
Q_{j}=\frac{a_{j}+K}{c_{j}+2K},
\end{align*}
where $c_{j}$ is the number of times worker $j$ has given a label, $a_{j}$ represents how many of those labels were consistent with the weighted vote, and $K$ is a smoothing parameter. In experiments, we set $K$ at 100, which is empirically shown by \citet{ertekin14} as the best choice. Then a confidence score is calculated as
\begin{align*}
\textrm{Score}(S_i)=\sum_{j\in S_i}y_{i,j}Q_{j},
\end{align*}
where $S_i$ is a subset of workers initially containing two top-quality workers and another random worker. Then CrowdSense determines whether to add a new worker $l$ to $S_i$ according to the criteria:
\begin{align*}
\frac{|\textrm{Score}(S_i)|-Q_{l}}{|S_i|+1}<\epsilon,
\end{align*}
where $\epsilon$ is a pre-determined threshold parameter. If the above is true, the candidate worker is added to $S_i$, and Score$(S_i)$ is re-calculated. After fixing the subset of workers for task $i$, CrowdSense aggregates the collected labels by the weighted voting mechanism, using $Q_{j}$ as the weight of worker $j$. 

\subsection{Optimistic Knowledge Gradient}

Optimistic Knowledge Gradient (OptKG) \citep{chen13} uses an $N$-coin-tossing model, formulates the task assignment problem as a Bayesian Markov decision process (MDP), and obtains the optimal allocation sequence for any finite budget $T$ via a computationally efficient approximate policy.

For better illustration, a simplified model with noiseless workers is first considered. In the model, task $i$ is characterized by $\theta_i$ drawn from a known Beta prior distribution, $\text{Beta}(a_i^0,b_i^0)$. In practice, we may simply start from $a_i^0=b_i^0=1$ if there is no prior knowledge \citep{chen13}. At each stage $t$ with $\text{Beta}(a_i^t,b_i^t)$ as the current posterior distribution for $\theta_i$, task $i_t$ is picked, and its label, drawn from Bernoulli distribution $y_{i_t}\sim \text{Bernoulli}(\theta_{i_t})$, is acquired. Then $\{a_i^t,b_i^t\}_{i=1}^N$ are put into an $n\times 2$ matrix $S^t$, called a state matrix with row $i$ as $S_i^t=(a_i^t,b_i^t)$. The state matrix is updated according to the following rule:
\begin{align*}
S^{t+1}=\begin{cases}
S^t+(\ve_{i_t},\vZero) &\text{if } y_{i_t}=1,\\
S^t+(\vZero,\ve_{i_t}) &\text{if } y_{i_t}=-1,
\end{cases}
\end{align*}
where $\ve_{i_t}$ is an $n$-dimensional vector with 1 at the $i_t$-th entry and 0 at all others. The state transition probability is calculated as
\begin{align*}
p(y_{i_t}=1\mid S^t,i_t)=\bE[\theta_{i_t}\mid S^t]=\frac{a_{i_t}^t}{a_{i_t}^t+b_{i_t}^t}.
\end{align*}
The stage-wise expected reward is defined as
\begin{align*}
R(S^t,i_t)=p_1R_1(a_{i_t}^{t},b_{i_t}^{t})+p_2R_2(a_{i_t}^{t},b_{i_t}^{t}),
\end{align*}
where
\begin{align*}
& p_1=p(y_{i_t}=1\mid S^t,i_t),\\
& p_2=1-p(y_{i_t}=1\mid S^t,i_t),\\
& R_1(a_{i_t}^{t},b_{i_t}^{t})=h(I(a_{i_t}^{t}+1,b_{i_t}^{t}))-h(I(a_{i_t}^{t},b_{i_t}^{t})),\\ & R_2(a_{i_t}^{t},b_{i_t}^{t})=h(I(a_{i_t}^{t},b_{i_t}^{t}+1))-h(I(a_{i_t}^{t},b_{i_t}^{t})),\\
& I(a_{i_t}^{t},b_{i_t}^{t})=p(\theta\geq0.5\mid\theta\sim\text{Beta}(a_{i_t}^{t},b_{i_t}^{t})),\\
& h(x)=\max(x,1-x).
\end{align*}
Then the following optimization problem is solved:
\begin{align*}
V(S^0)=G_0(S^0)+\sup_\pi\bE_\pi\left[\sum_{t=0}^{T-1}R(S^t,i_t)\right],
\end{align*}
where $G_0(S^0)=\sum_{i=1}^{n}h(I(a_i^0,b_i^0))$. With the above definitions, the problem is formulated as a $T$-stage MDP associated with a tuple
\begin{align*}
\{T,\{\cS^t\},\cA,p(y_{i_t}\mid S^t,i_t),R(S^t,i_t)\},
\end{align*}
where
\begin{align*}
\cS^t=\left\{\{a_i^t,b_i^t\}_{i=1}^{n}:a_i^t\geq a_i^0,b_i^t\geq b_i^0,\sum_{i=1}^{n}(a_i^t-a_i^0)+(b_i^t-b_i^0)=t\right\},
\end{align*}
and the action space is the set of indices for tasks that could be labeled next: $\cA=\{1,\ldots,n\}$.

Since the computation of dynamic programming for obtaining the optimal policy of the above formulation is intractable, a computationally efficient and consistent approximate policy is employed, resulting in the criterion for task selection at each stage $t$:
\begin{align*}
i_t = \argmax_{i\in\{1,\ldots,n\}}\left(\max(R_1(a_i^t,b_i^t),R_2(a_i^t,b_i^t))\right).
\end{align*}
Finally, OptKG outputs the positive set
\begin{align*}
H_T=\{i:a_i^T\geq b_i^T\}.
\end{align*}

Based on the simplified model above, workers' reliability can be further introduced. Specifically, worker $j$ is characterized by a scalar parameter
\begin{align*}
\rho_j=p(y_{i,j}=y_i^*\mid y_i^*),
\end{align*}
and is assumed drawn from a Beta prior distribution: $\rho_j\sim\text{Beta}(c_j^0,d_j^0)$. Then at each stage, a pair of worker and task is adaptively selected, according to a new criterion, which is derived in Appendix.5 in the original OptKG paper \citep{chen13}. In experiments, we follow \citet{chen13} and set $c_j^0=4$ and $d_j^0=1$ for each worker, which indicates that we have the prior belief that the average accuracy of workers is $4/5=80\%$.

\subsection{Discussion}
Although IEThresh and CrowdSense employ different criterions for dealing with the exploration-exploitation trade-off in worker selection, they share the similar mechanism of dynamically sampling worker subsets. In particular, at each step, a new task comes, workers' reliability is learned according to their labeling performances on previous tasks, and a subset of workers considered to be reliable is sampled for labeling the new task. However, in heterogeneous crowdsourcing, a worker in the subset who is reliable at previous tasks may be bad at new ones. On the other hand, the worker who is good at new tasks may have already been eliminated from the subset. Therefore, it is reasonable to model the task-dependent reliability for workers, and then match each task to workers who can do it best.

Another issue in IEThresh and CrowdSense is that the exact amount of total budget is not pre-fixed. In these two methods, a threshold parameter $\epsilon$ is used for controlling the size of worker subsets. That is, $\epsilon$ determines the total budget. However, how many workers in each subset is unknown beforehand. Therefore, we will not know the exact amount of total budget until the whole task assignment process ends. This is not so practical in crowdsourcing, since in the task assignment problem, we pursue not only collecting reliable labels but also intelligently using the pre-fixed budget. Moreover, both of IEThresh and CrowdSense lack of theoretical analyses about the relation between the budget and the performance.

OptKG and the proposed method (BBTA) attempt to intelligently use a pre-fixed budget, and ask one worker for a label of the current task at each step. OptKG formulates the task assignment problem as a Bayesian MDP, and is proved to produce a consistent policy in homogeneous worker setting (i.e. the policy will achieve 100\% accuracy almost surely when the total budget goes to infinity). However, when the heterogeneous reliability of workers is introduced, more sophisticated approximation is also involved in OptKG, making it difficult to give theoretical analysis for OptKG in heterogeneous worker setting. On the other hand, BBTA is a contextual bandit formulation designed for heterogeneous crowdsourcing, and the regret analysis demonstrates the performance of our task assignment strategy will converge to that of the optimal one when the total budget goes to infinity.

\section{Experiments}
\label{sec:experiment}

In this section, we experimentally evaluate the usefulness of the proposed bandit-based task assignment (BBTA) method. To compare BBTA with existing methods, we first conduct experiments on benchmark data with simulated workers, and then use real data for further comparison. All of the experimental results are averaged over 30 runs.

\subsection{Benchmark Data}

\begin{figure*}[t]
\centering
\subfigure[$N=351,K=30,S=3$]{
\includegraphics[width=0.31\textwidth]{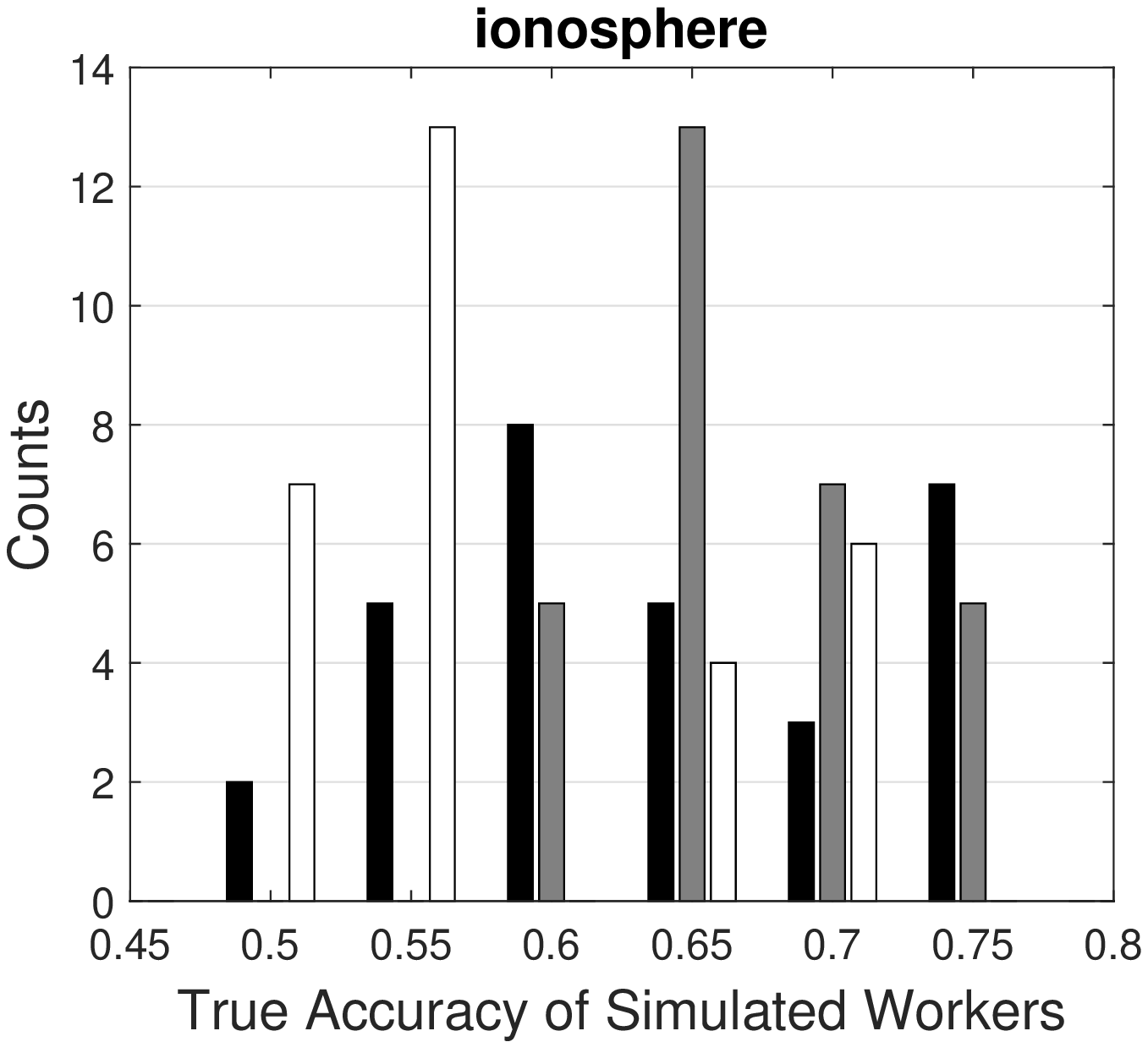}
\label{fig:ionosphere-workers}
}
\subfigure[$N=569,K=40,S=4$]{
\includegraphics[width=0.31\textwidth]{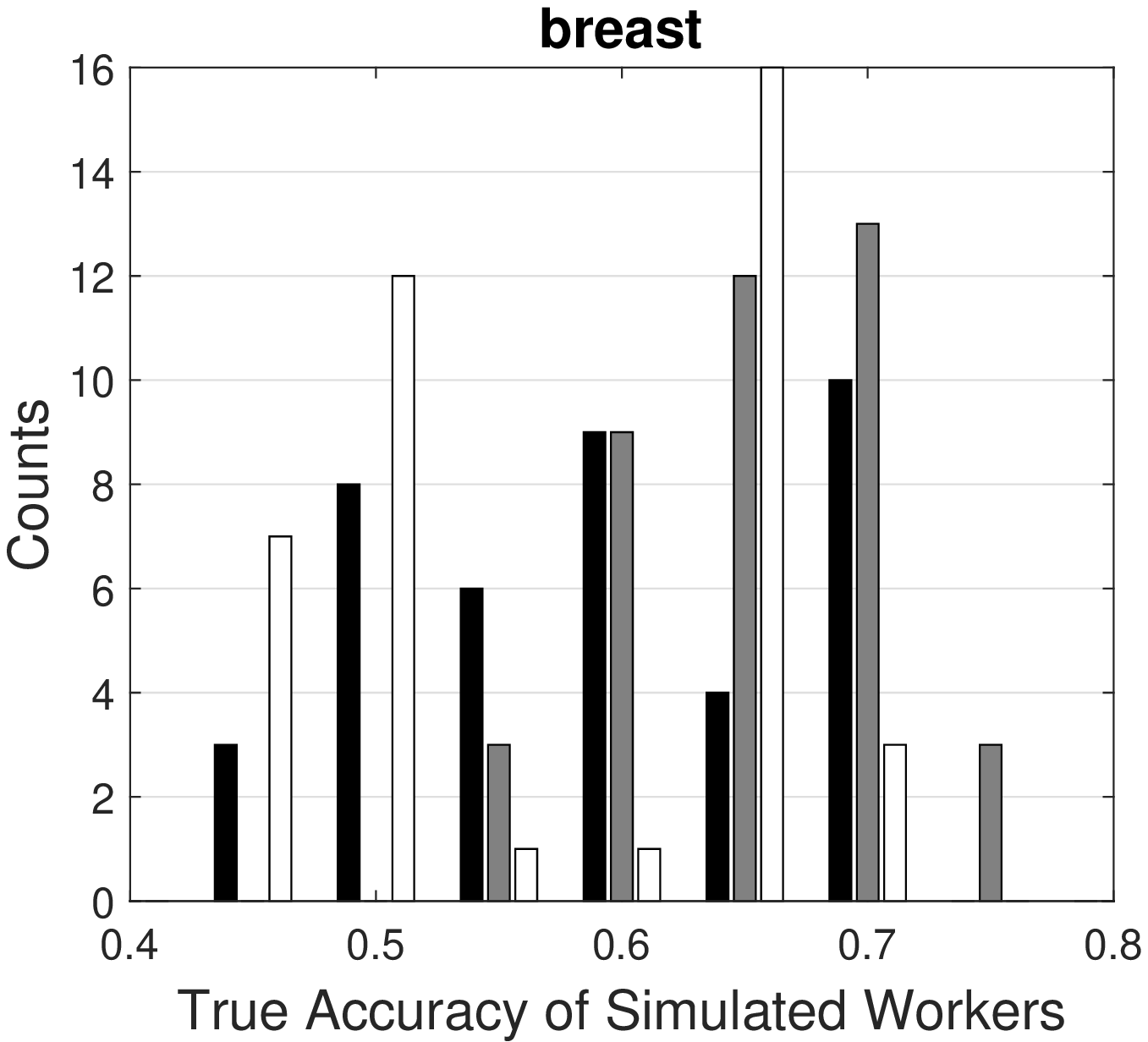}
\label{fig:breast-workers}
}
\subfigure[$N=768,K=50,S=5$]{
\includegraphics[width=0.31\textwidth]{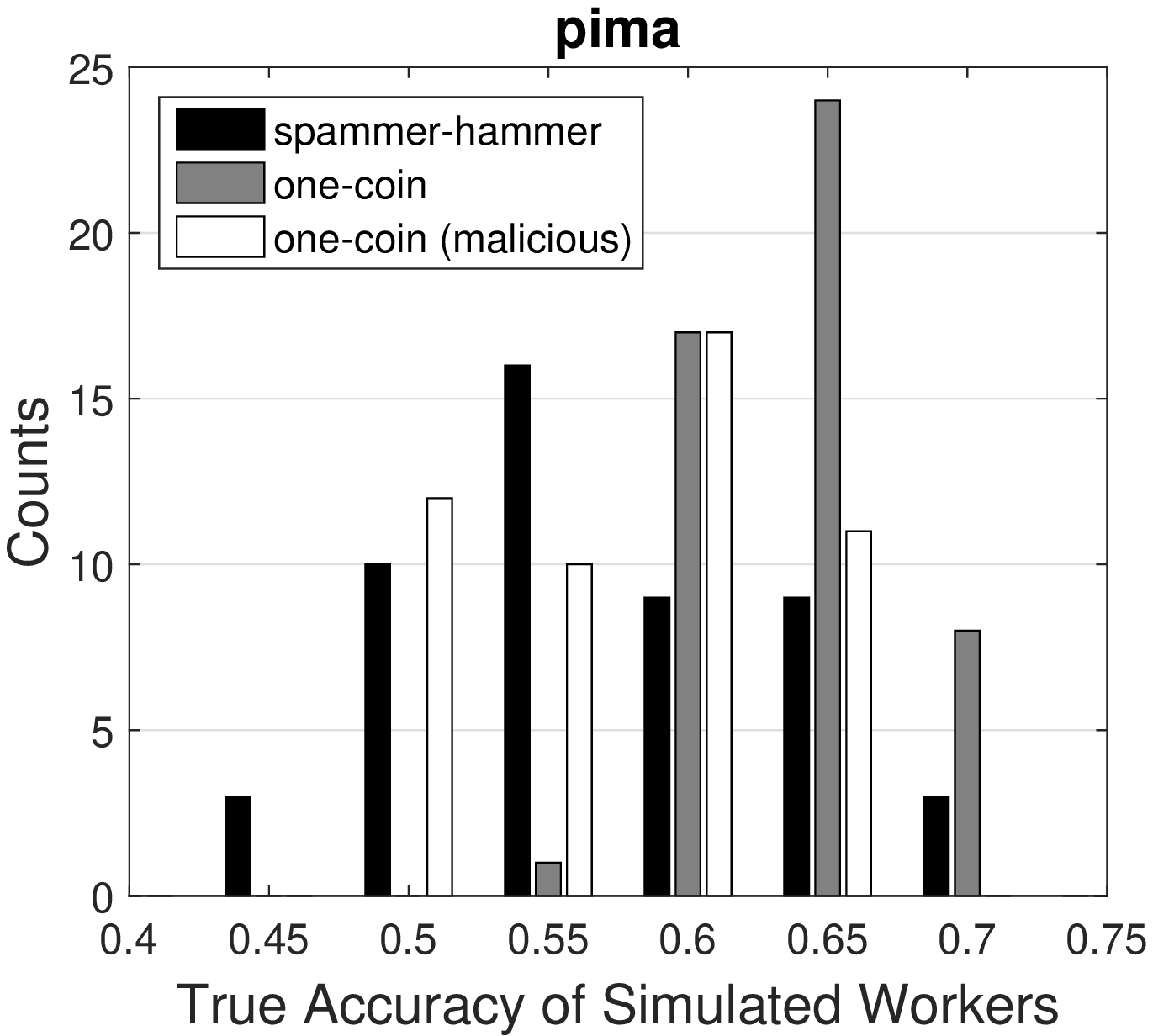}
\label{fig:pima-workers}
}
\caption{Distribution of true accuracy of simulated workers for three benchmark datasets with three worker models.}
\label{fig:bench-workers}
\end{figure*}

\begin{figure*}[t]
\centering
\subfigure[$N=351,K=30,S=3$]{
\includegraphics[width=0.31\textwidth]{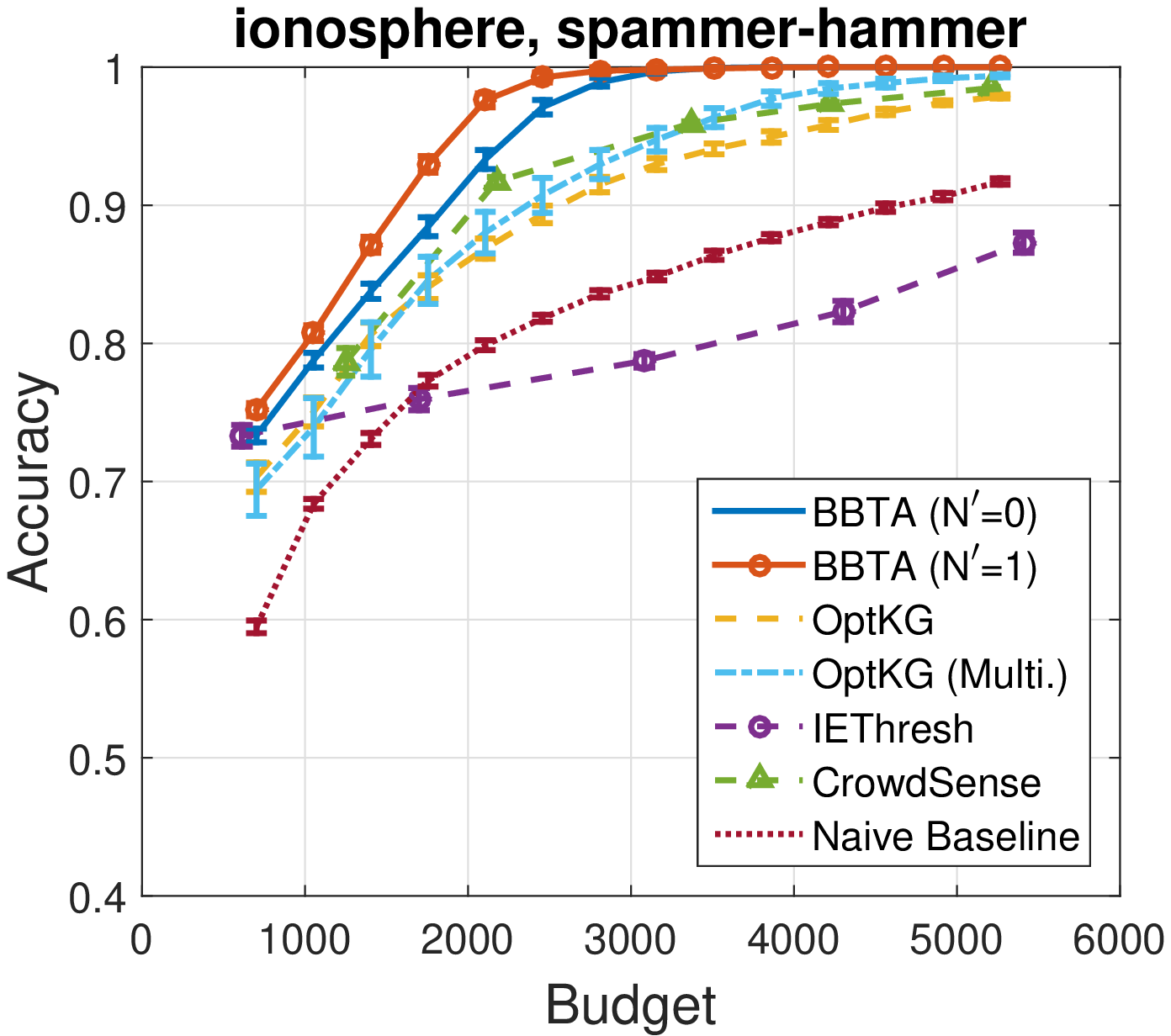}
\label{fig:ionosphere-spammer-acc}
}
\subfigure[$N=569,K=40,S=4$]{
\includegraphics[width=0.31\textwidth]{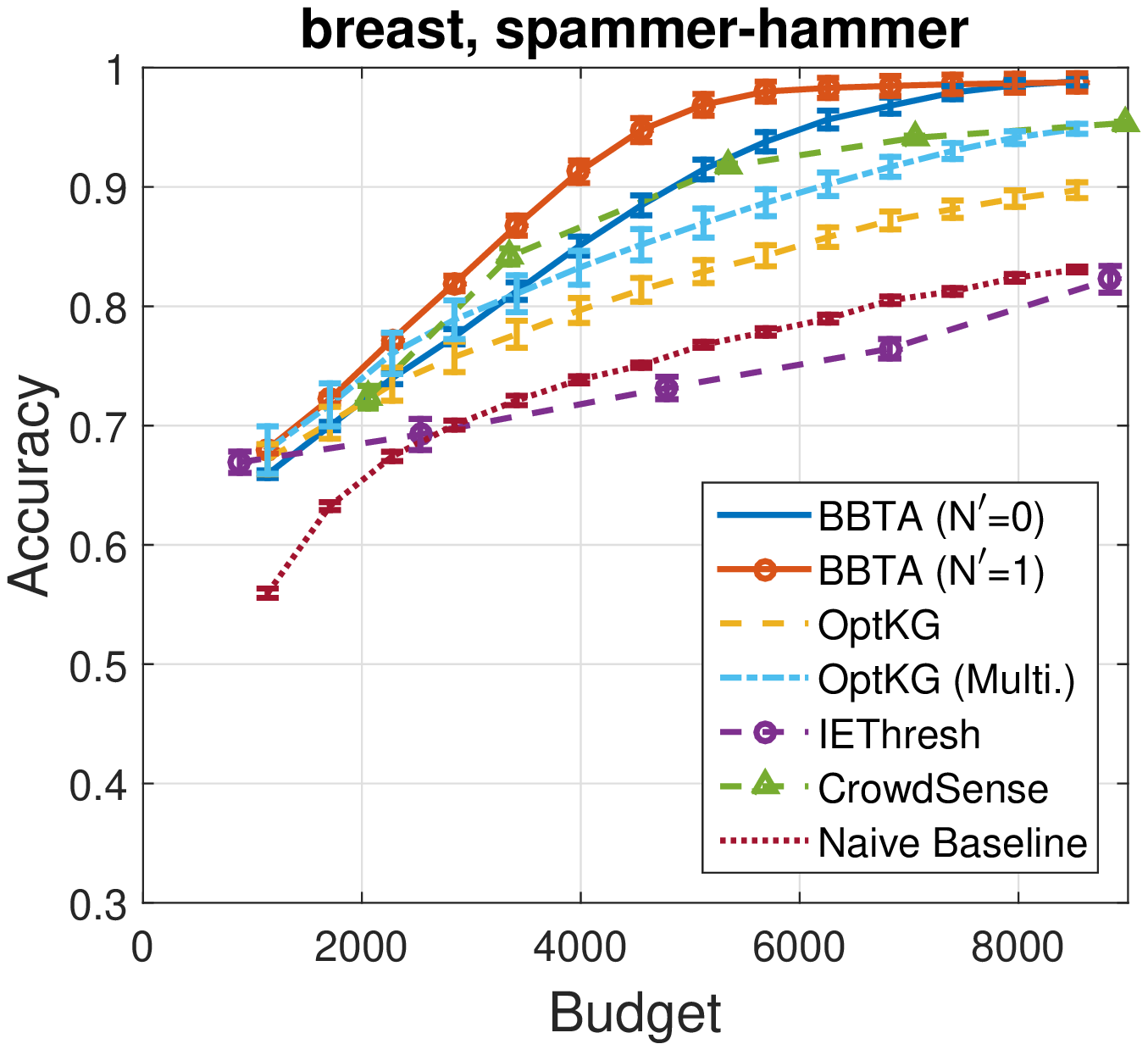}
\label{fig:breast-spammer-acc}
}
\subfigure[$N=768,K=50,S=5$]{
\includegraphics[width=0.31\textwidth]{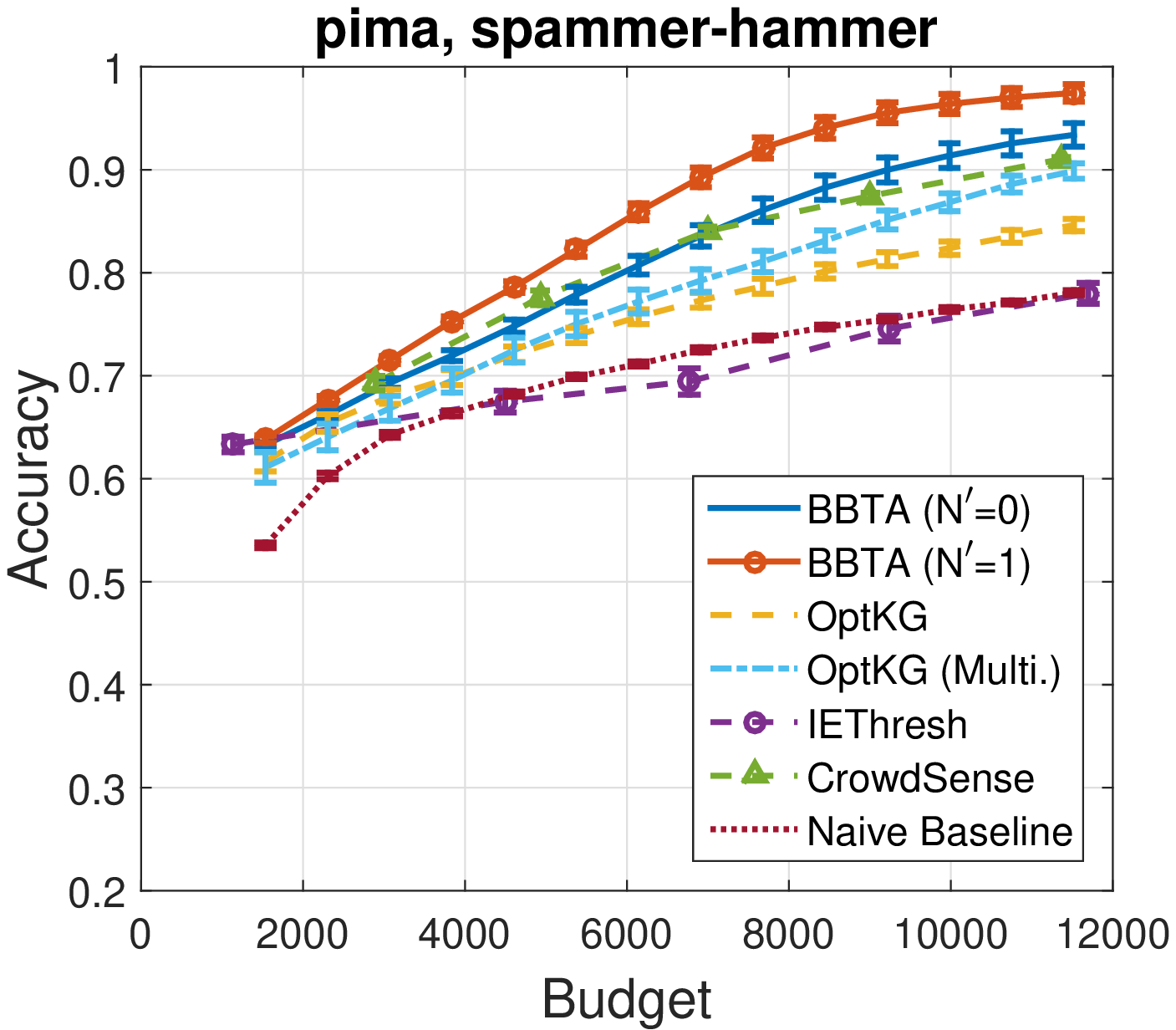}
\label{fig:pima-spammer-acc}
}
\subfigure[$N=351,K=30,S=3$]{
\includegraphics[width=0.31\textwidth]{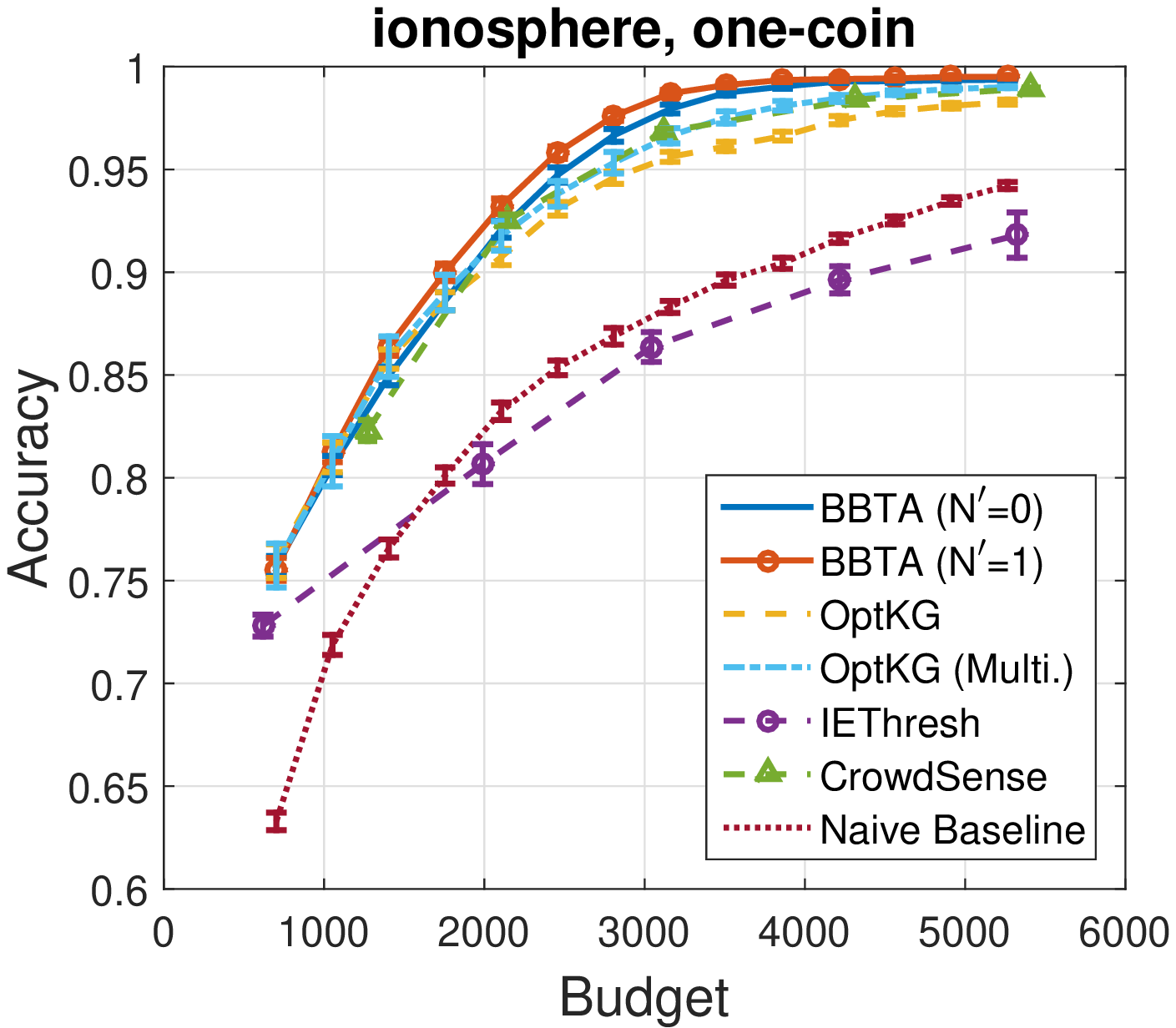}
\label{fig:ionosphere-onecoin-acc}
}
\subfigure[$N=569,K=40,S=4$]{
\includegraphics[width=0.31\textwidth]{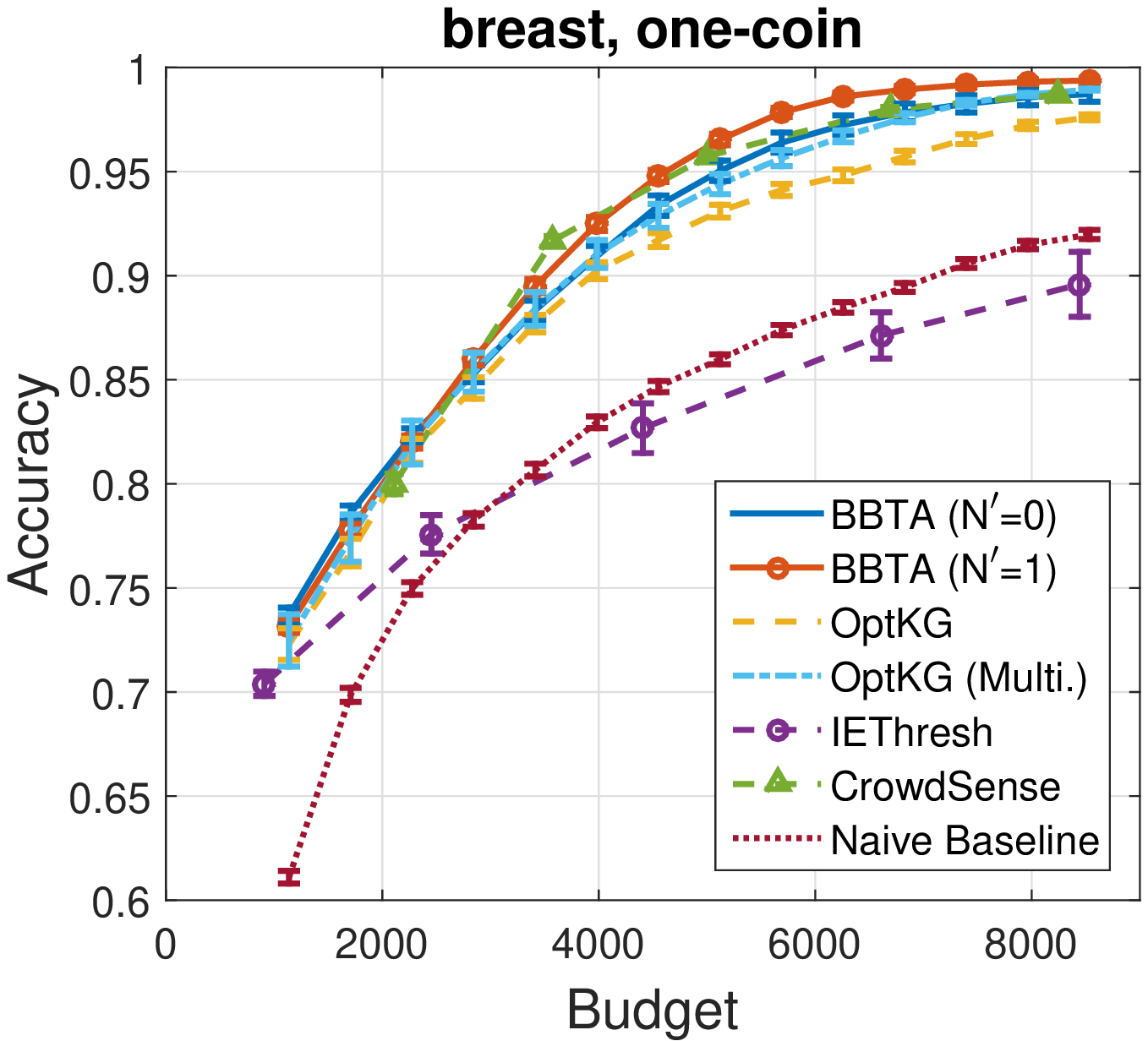}
\label{fig:breast-onecoin-acc}
}
\subfigure[$N=768,K=50,S=5$]{
\includegraphics[width=0.31\textwidth]{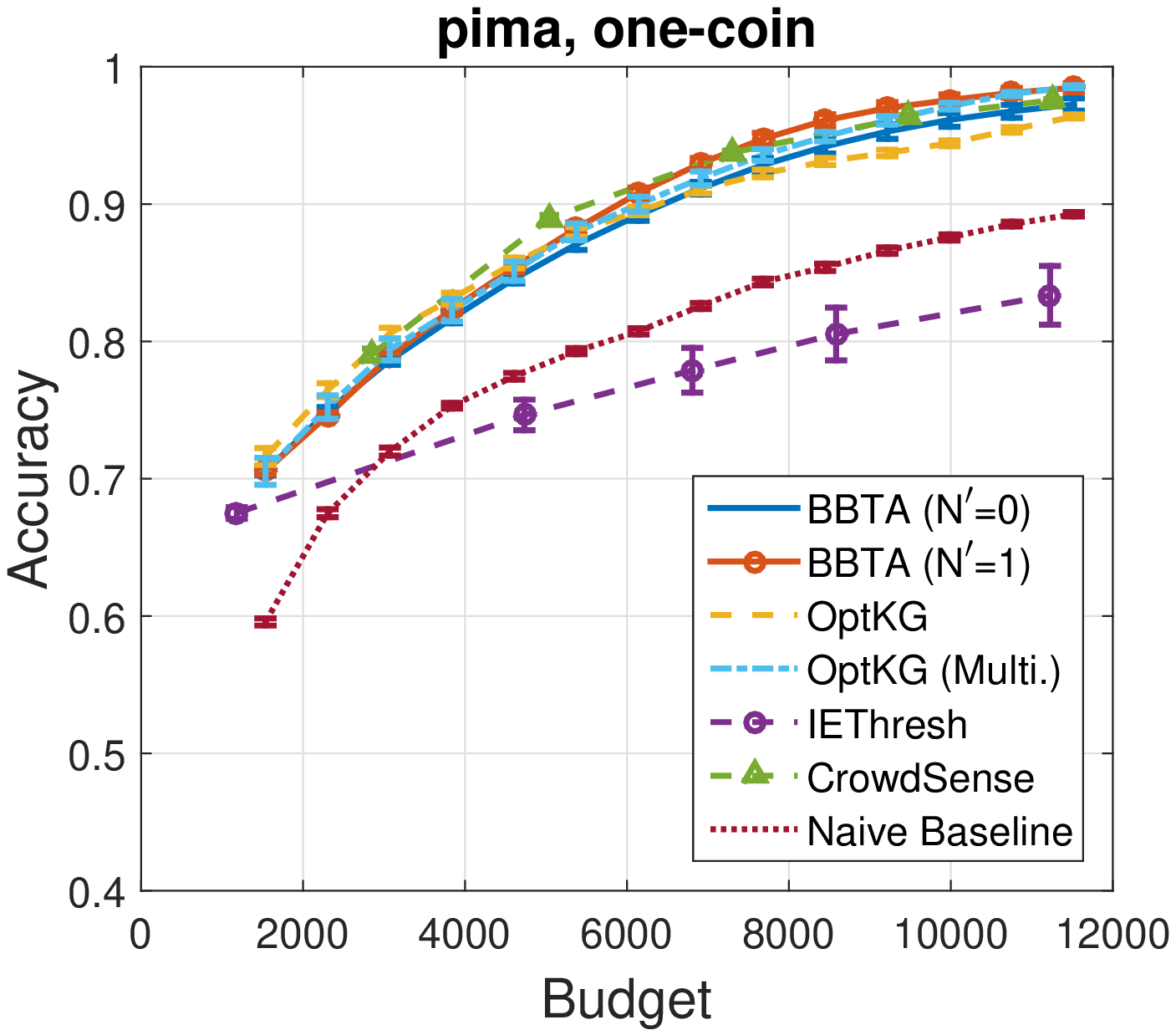}
\label{fig:pima-onecoin-acc}
}
\subfigure[$N=351,K=30,S=3$]{
\includegraphics[width=0.31\textwidth]{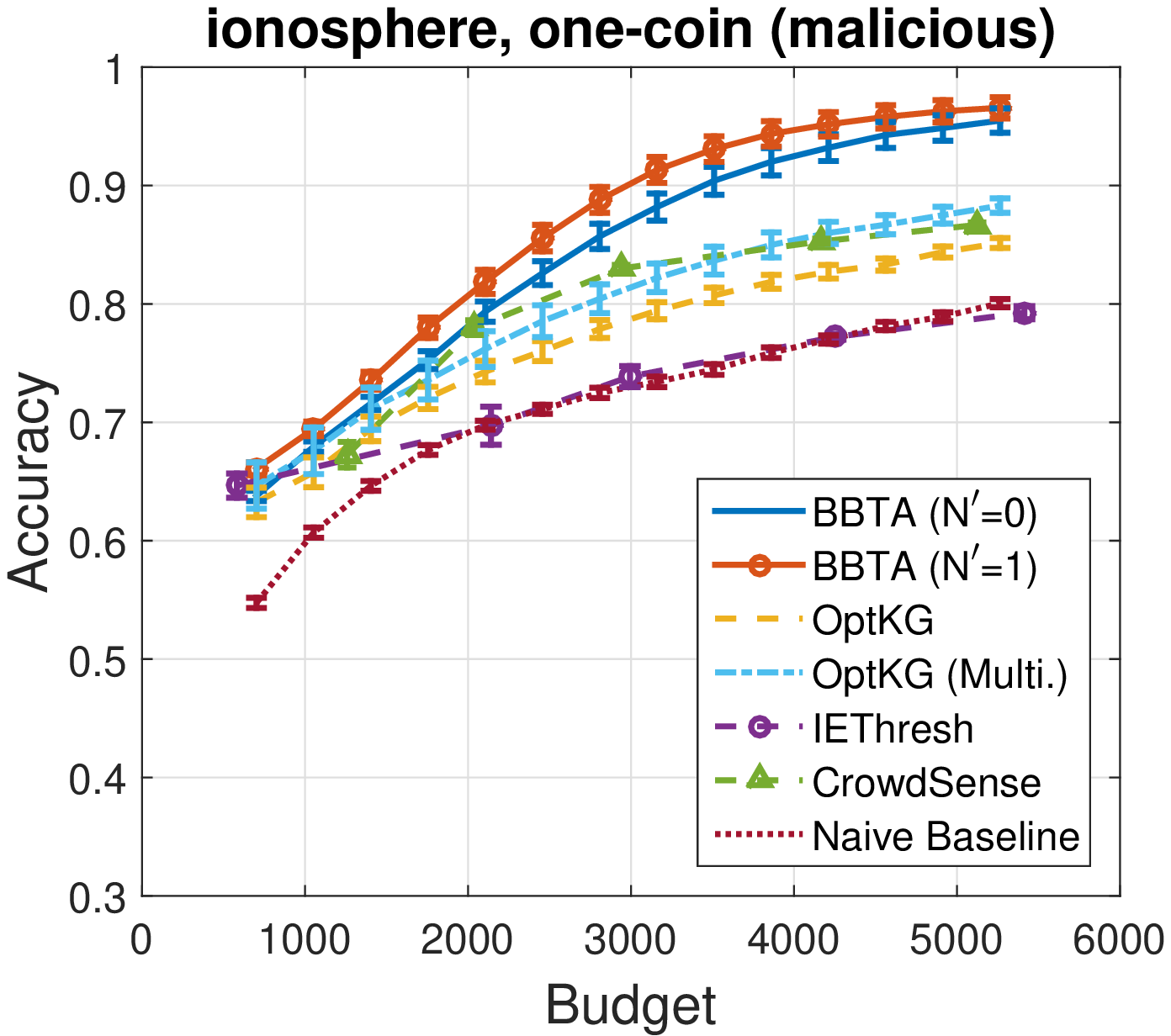}
\label{fig:ionosphere-onecoin_m-acc}
}
\subfigure[$N=569,K=40,S=4$]{
\includegraphics[width=0.31\textwidth]{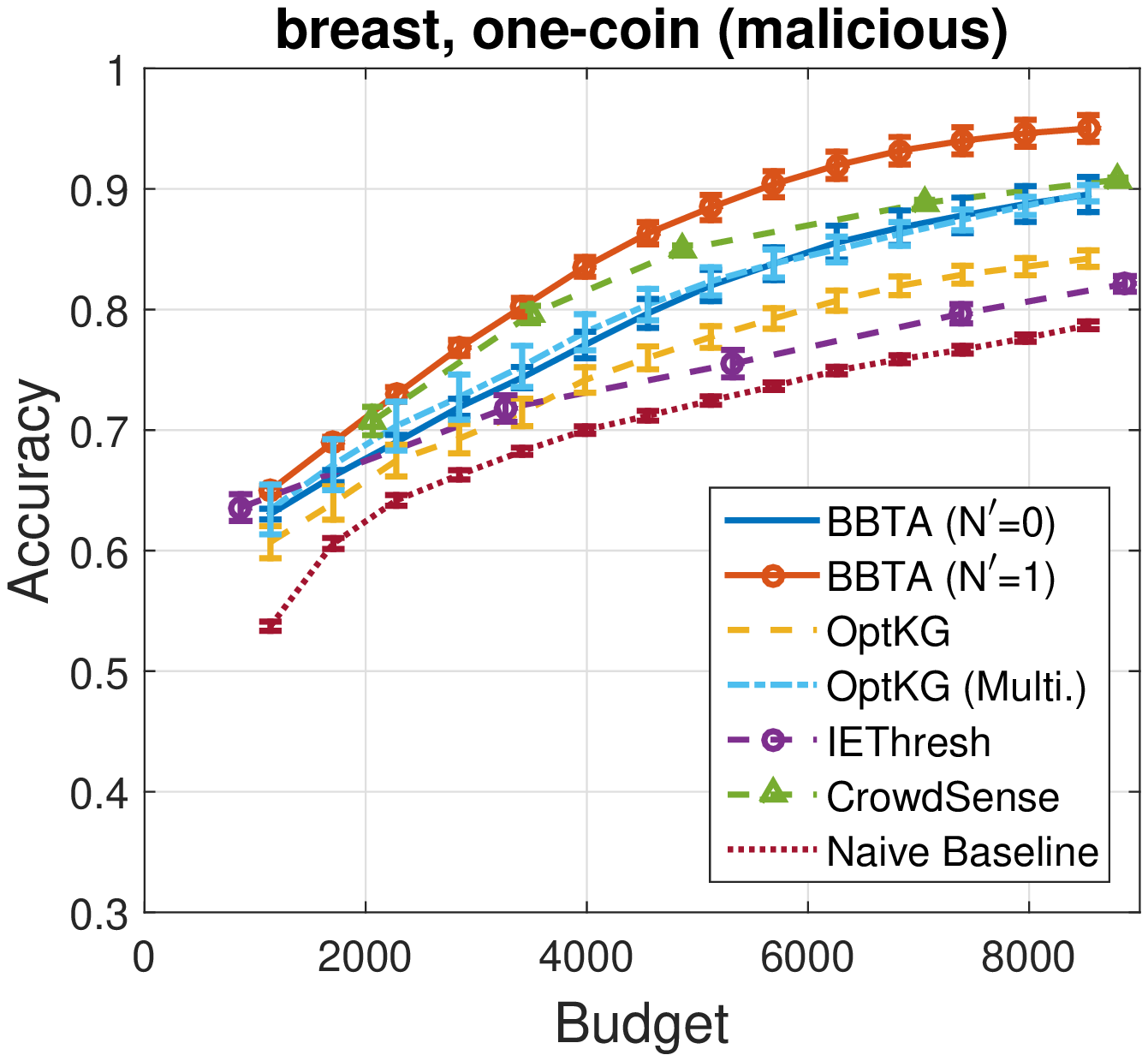}
\label{fig:breast-onecoin_m-acc}
}
\subfigure[$N=768,K=50,S=5$]{
\includegraphics[width=0.31\textwidth]{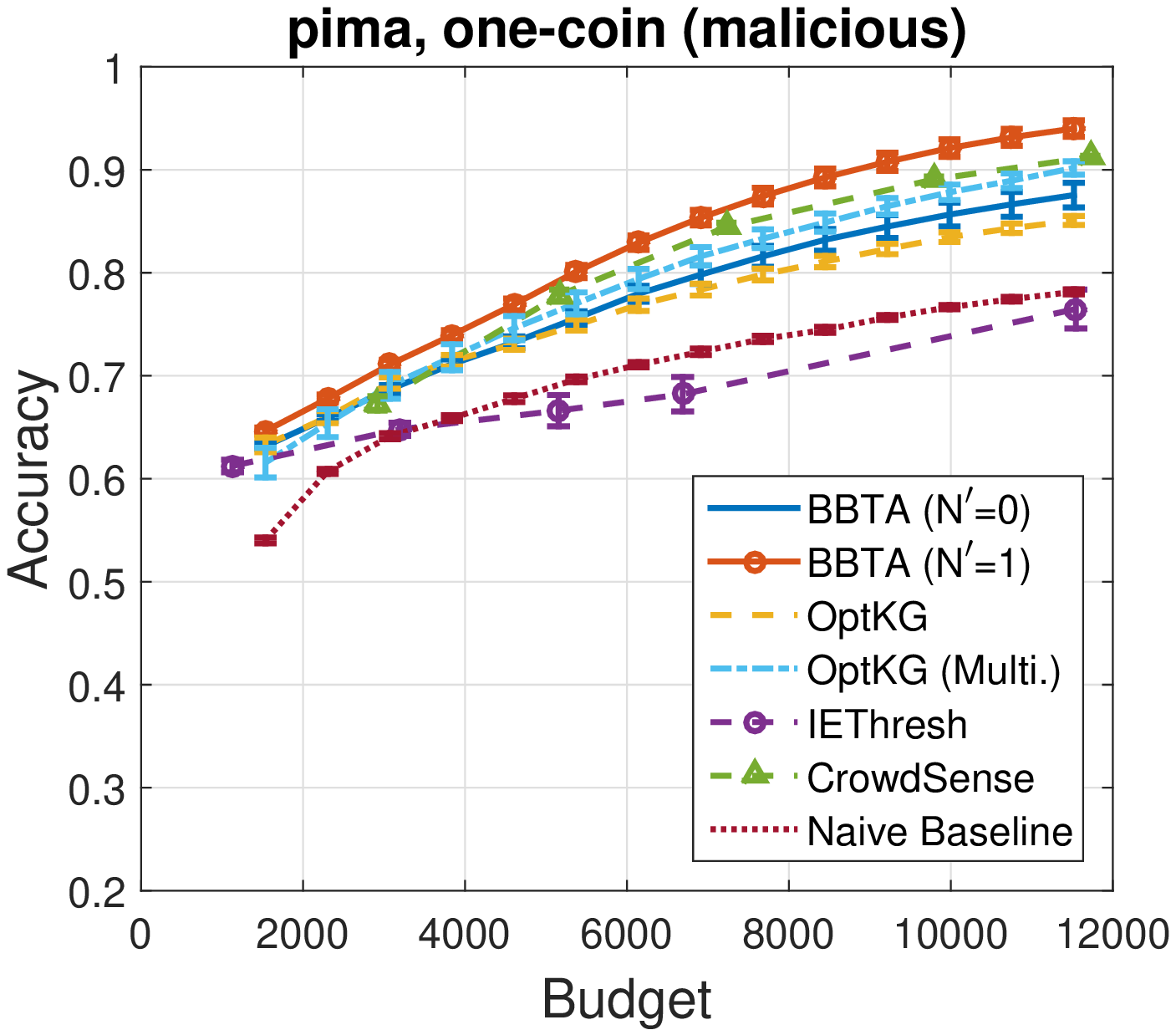}
\label{fig:pima-onecoin_m-acc}
}
\caption{Comparison results on three benchmark datasets with three worker models.}
\label{fig:bench-acc}
\end{figure*}

\begin{figure*}[t]
\centering
\subfigure[$N=351,K=30,S=3$]{
\includegraphics[width=0.31\textwidth]{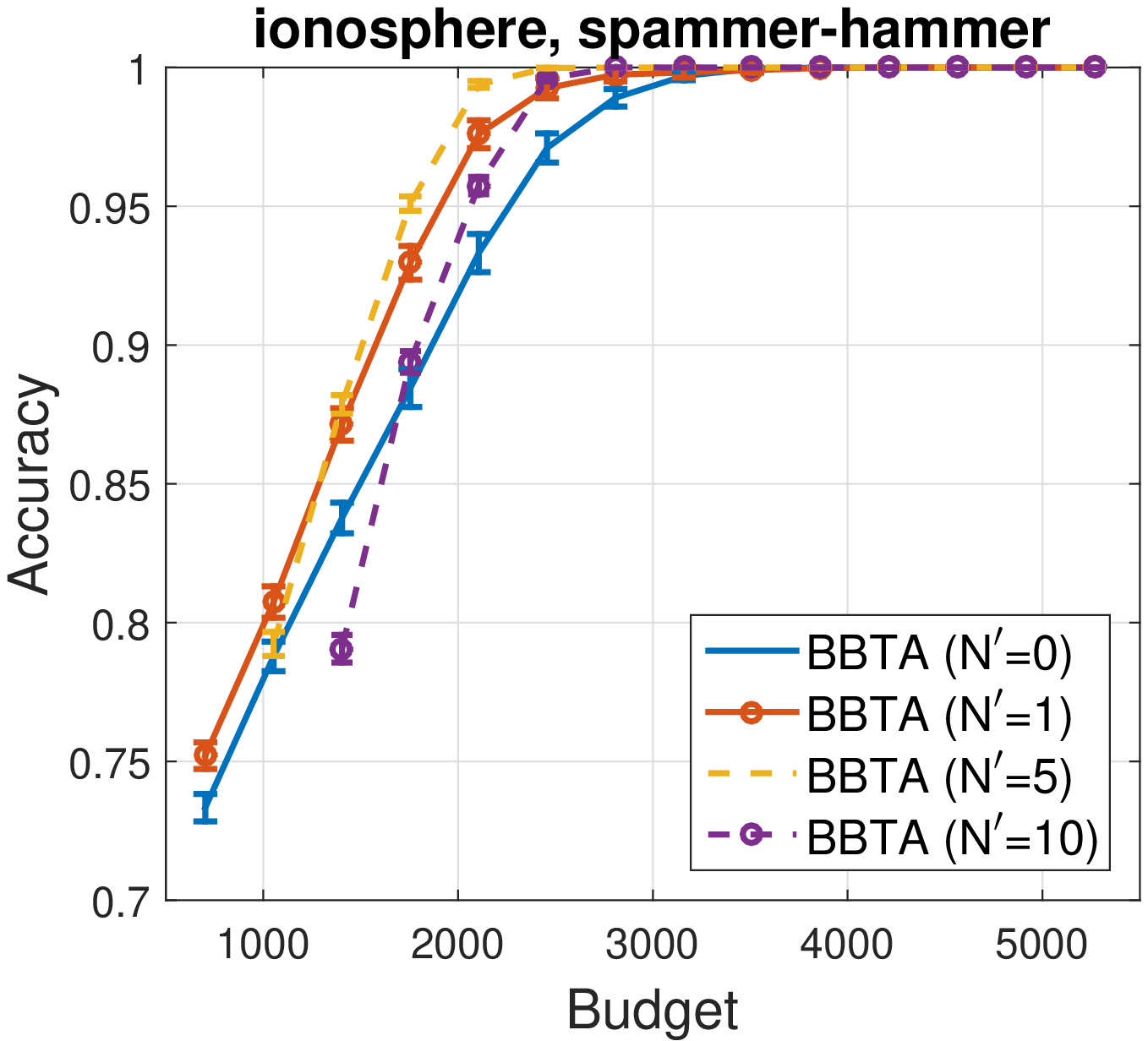}
\label{fig:ionosphere-spammer-nprime}
}
\subfigure[$N=569,K=40,S=4$]{
\includegraphics[width=0.31\textwidth]{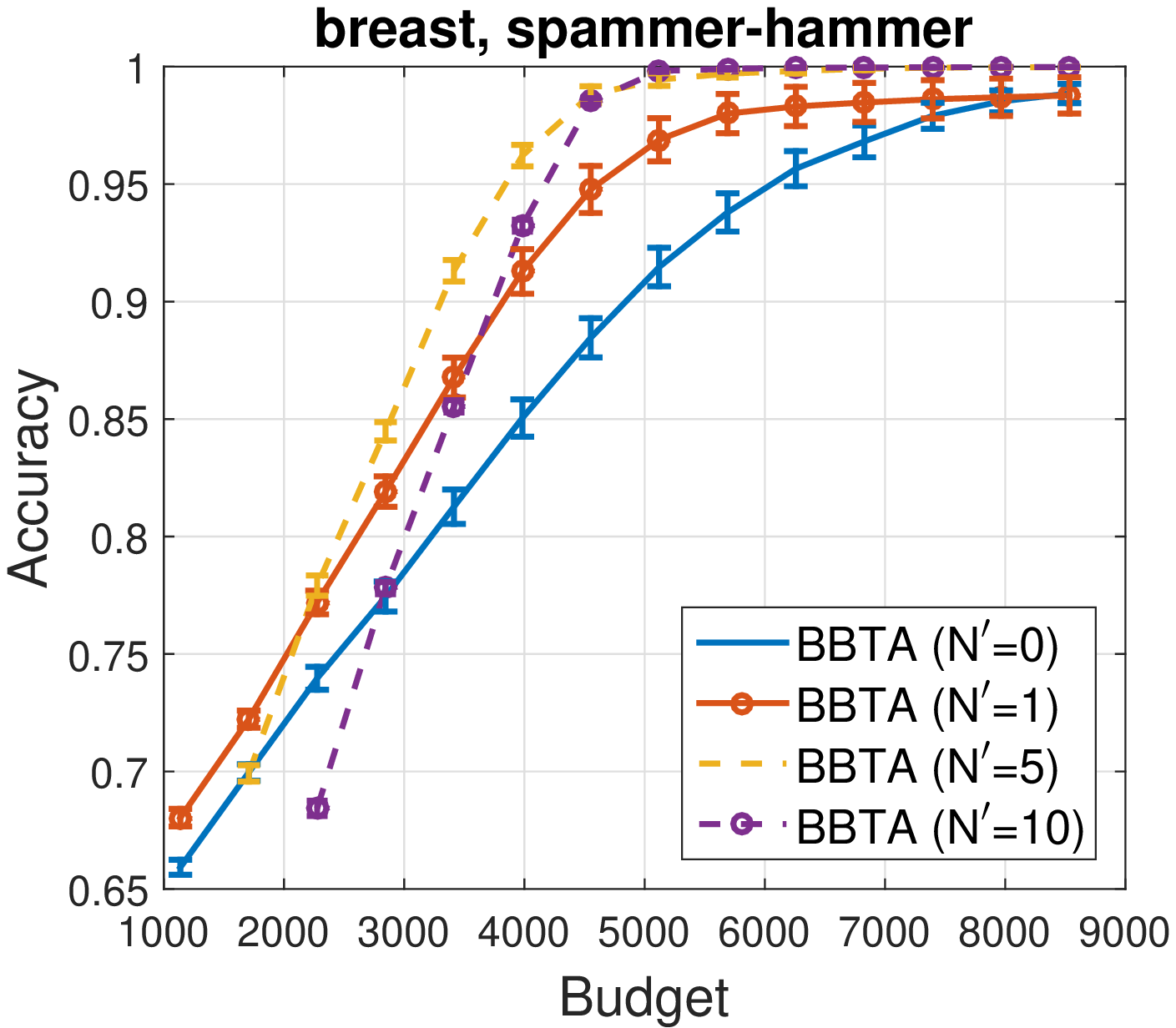}
\label{fig:breast-spammer-nprime}
}
\subfigure[$N=768,K=50,S=5$]{
\includegraphics[width=0.31\textwidth]{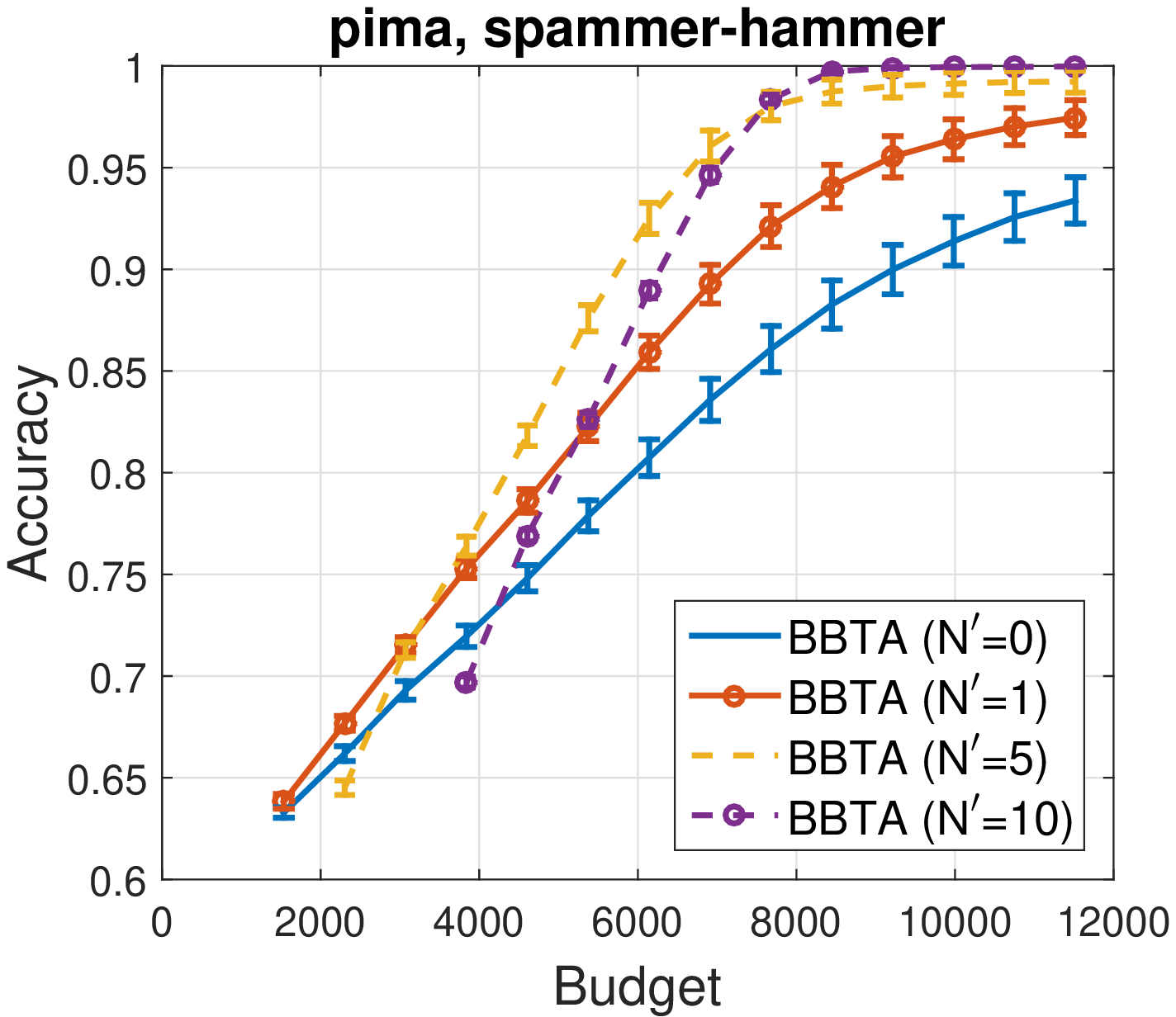}
\label{fig:pima-spammer-nprime}
}
\subfigure[$N=351,K=30,S=3$]{
\includegraphics[width=0.31\textwidth]{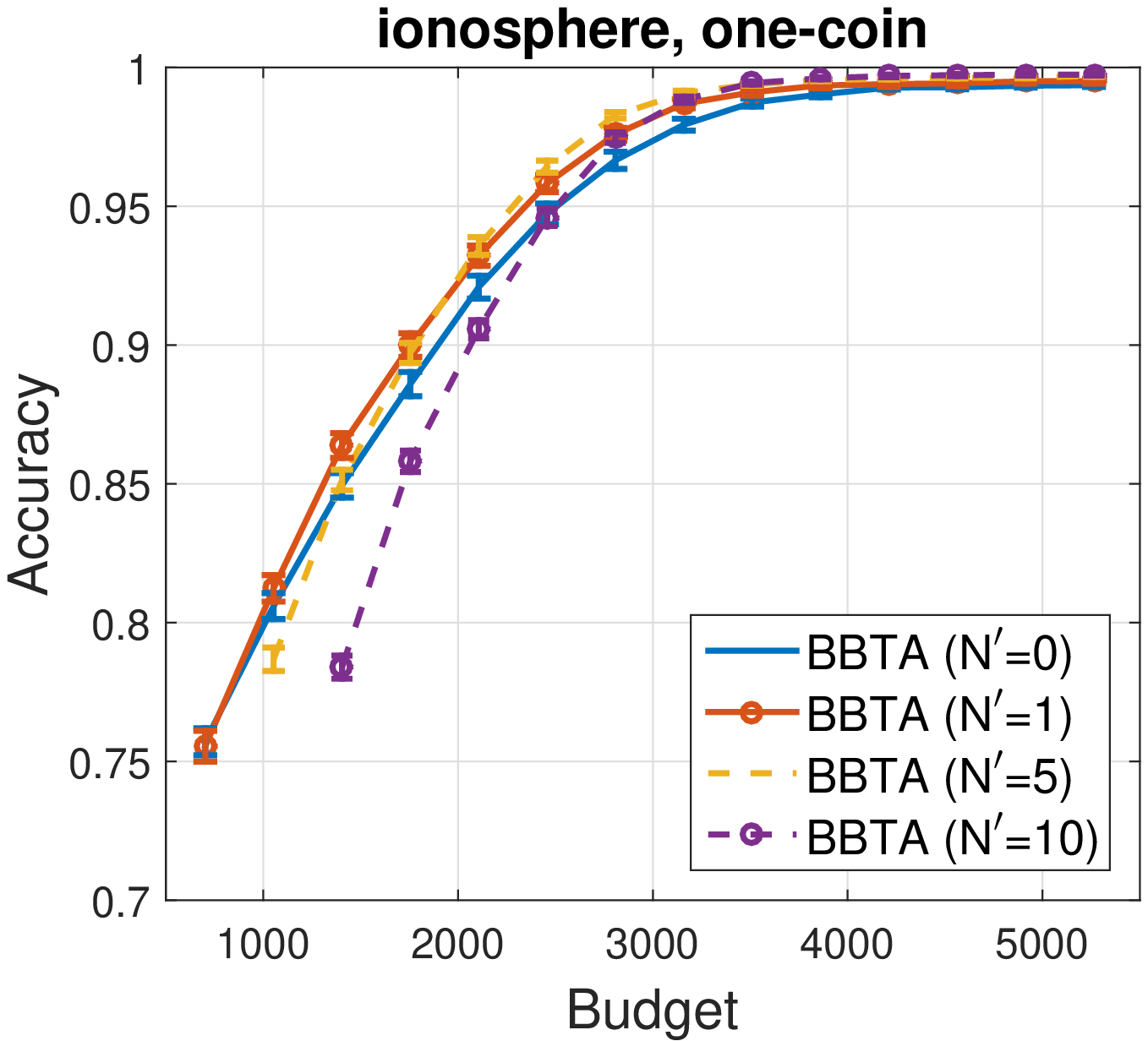}
\label{fig:ionosphere-onecoin-nprime}
}
\subfigure[$N=569,K=40,S=4$]{
\includegraphics[width=0.31\textwidth]{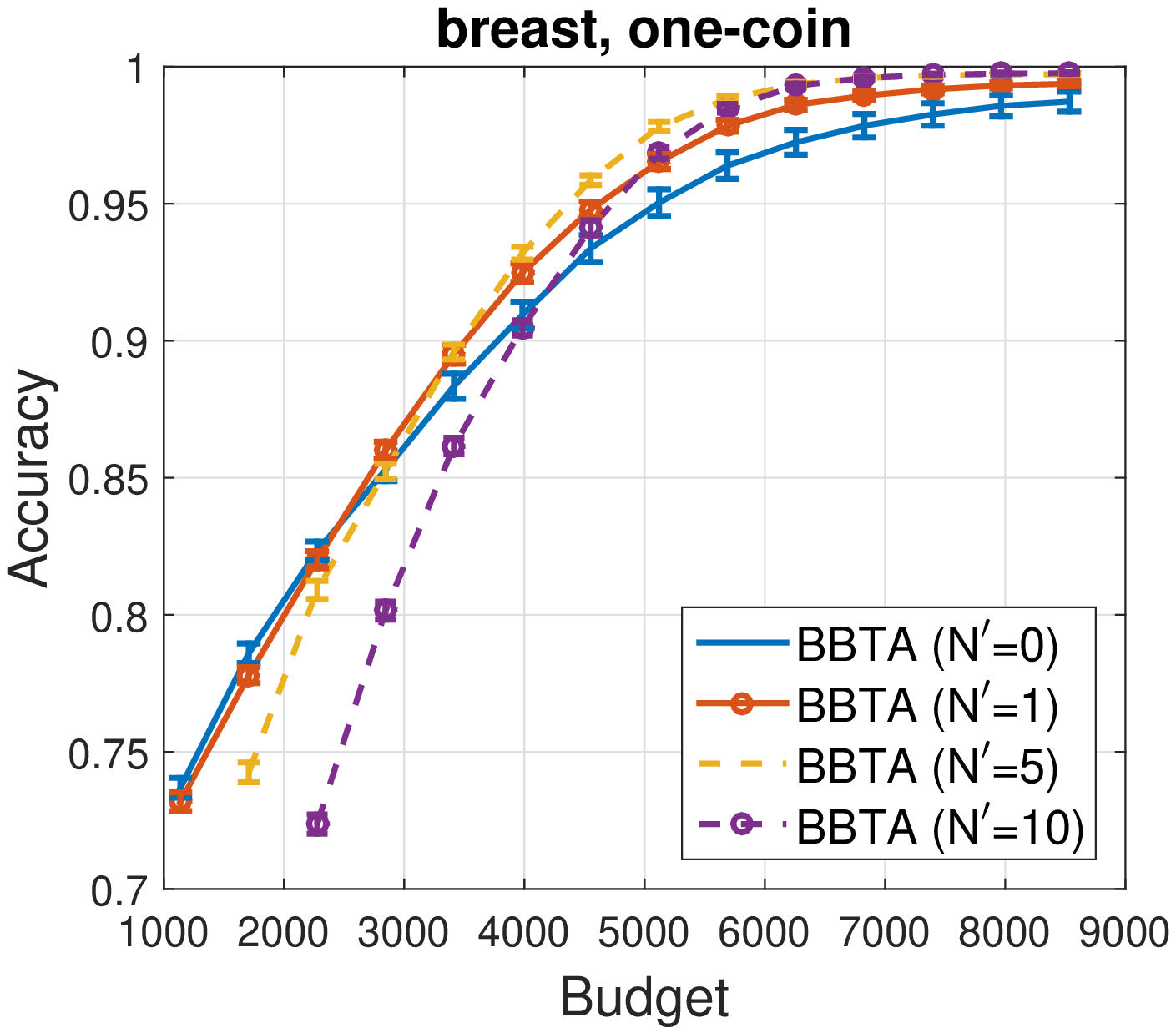}
\label{fig:breast-onecoin-nprime}
}
\subfigure[$N=768,K=50,S=5$]{
\includegraphics[width=0.31\textwidth]{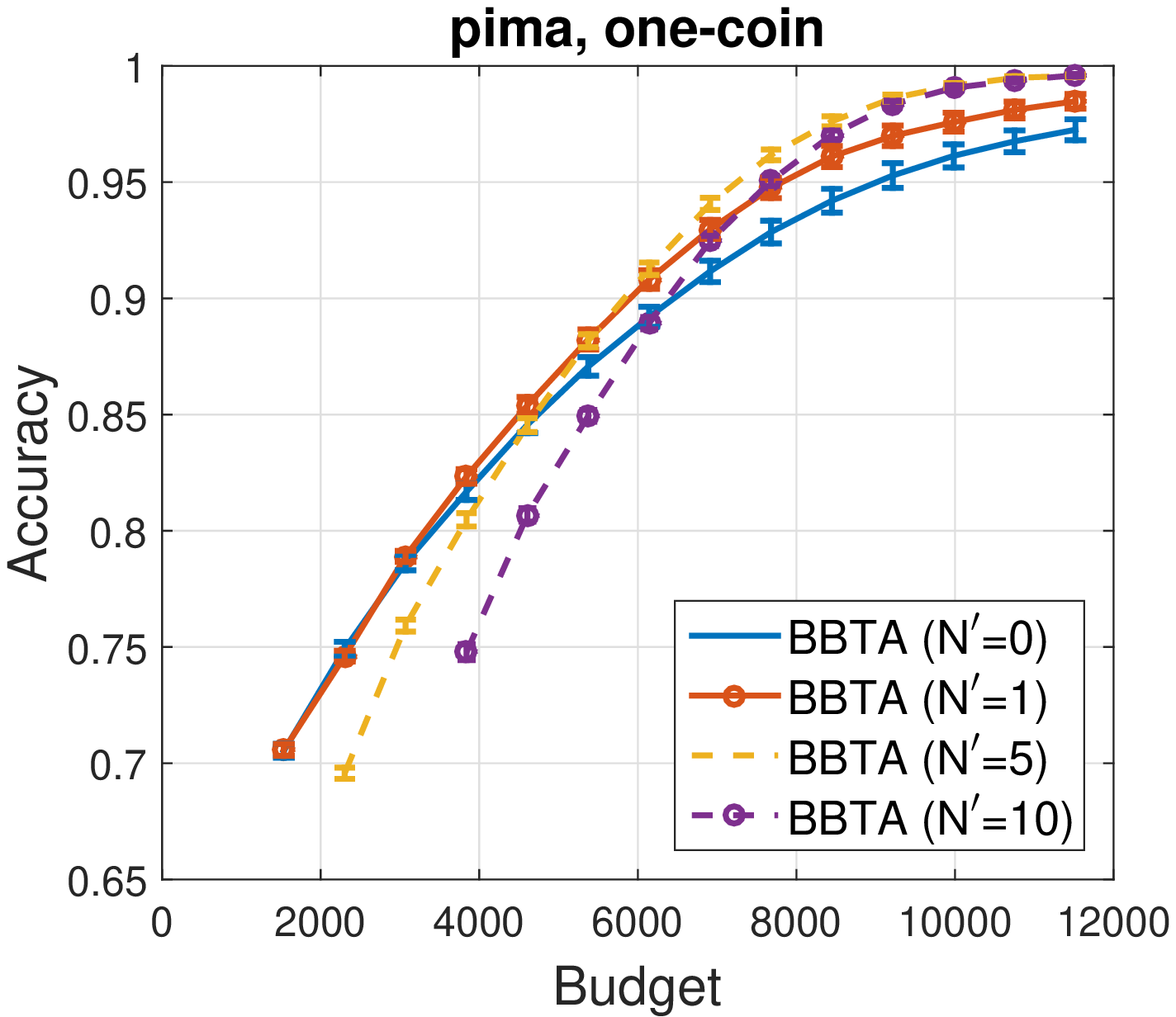}
\label{fig:pima-onecoin-nprime}
}
\subfigure[$N=351,K=30,S=3$]{
\includegraphics[width=0.31\textwidth]{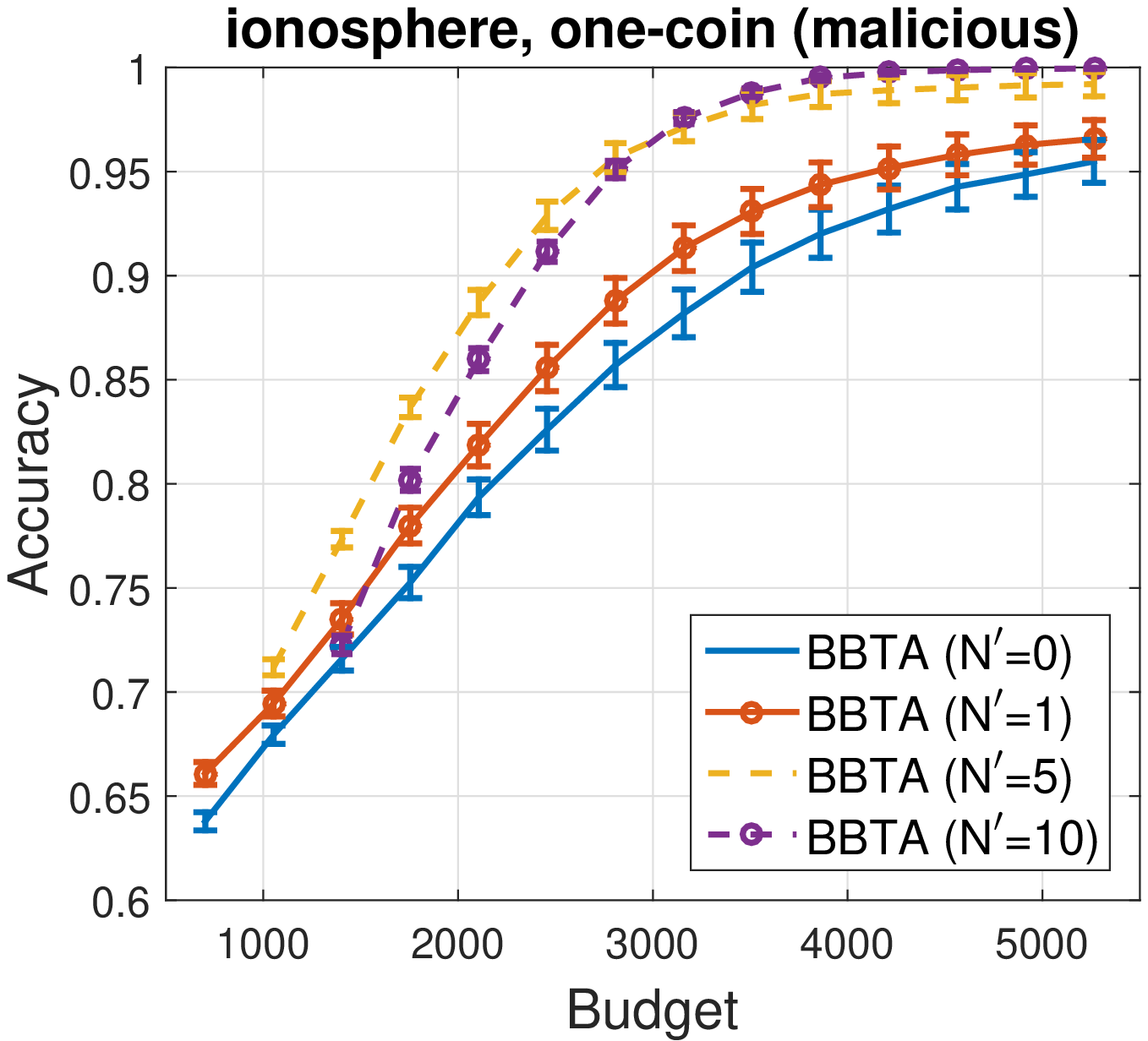}
\label{fig:ionosphere-onecoin_m-nprime}
}
\subfigure[$N=569,K=40,S=4$]{
\includegraphics[width=0.31\textwidth]{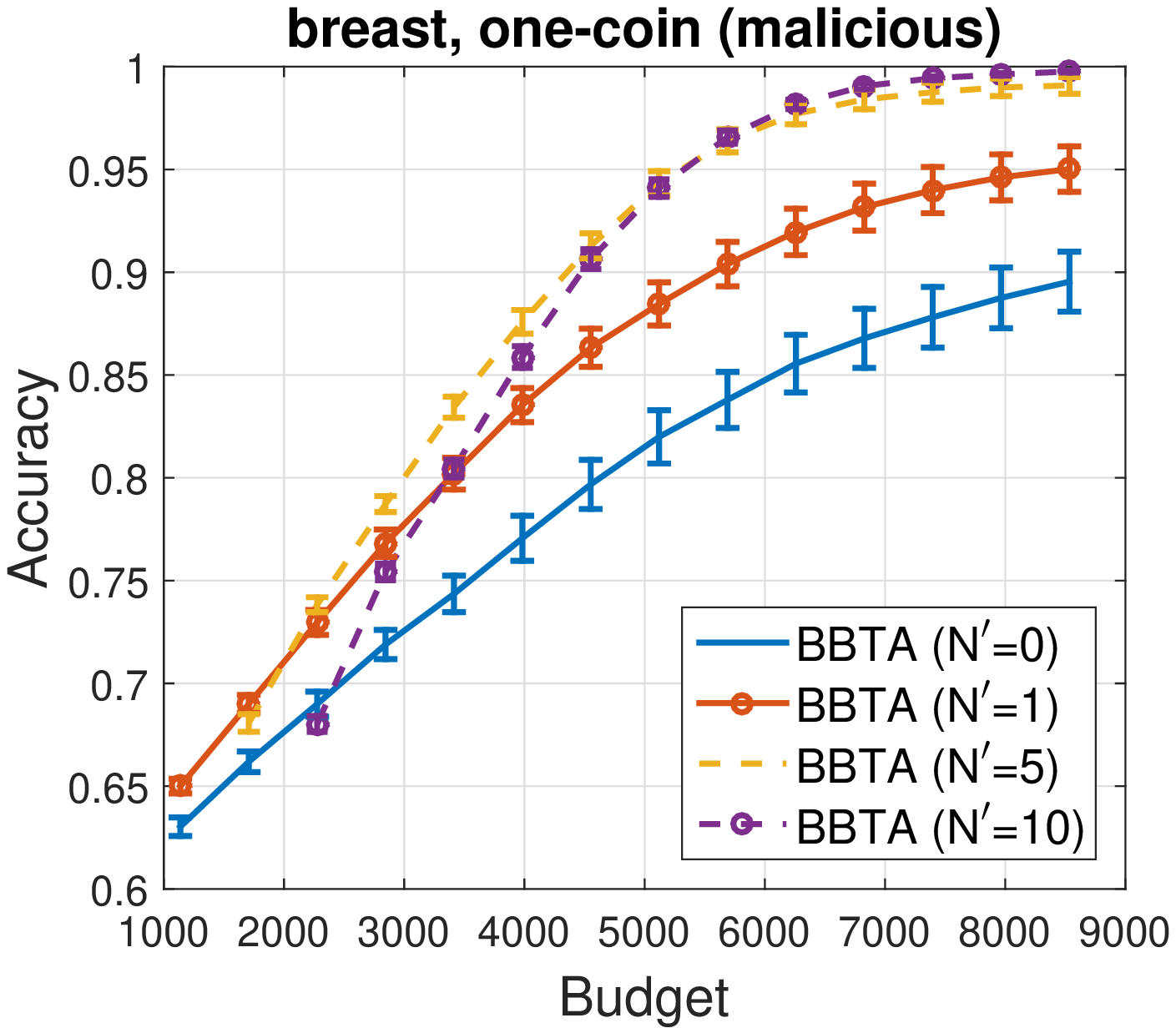}
\label{fig:breast-onecoin_m-nprime}
}
\subfigure[$N=768,K=50,S=5$]{
\includegraphics[width=0.31\textwidth]{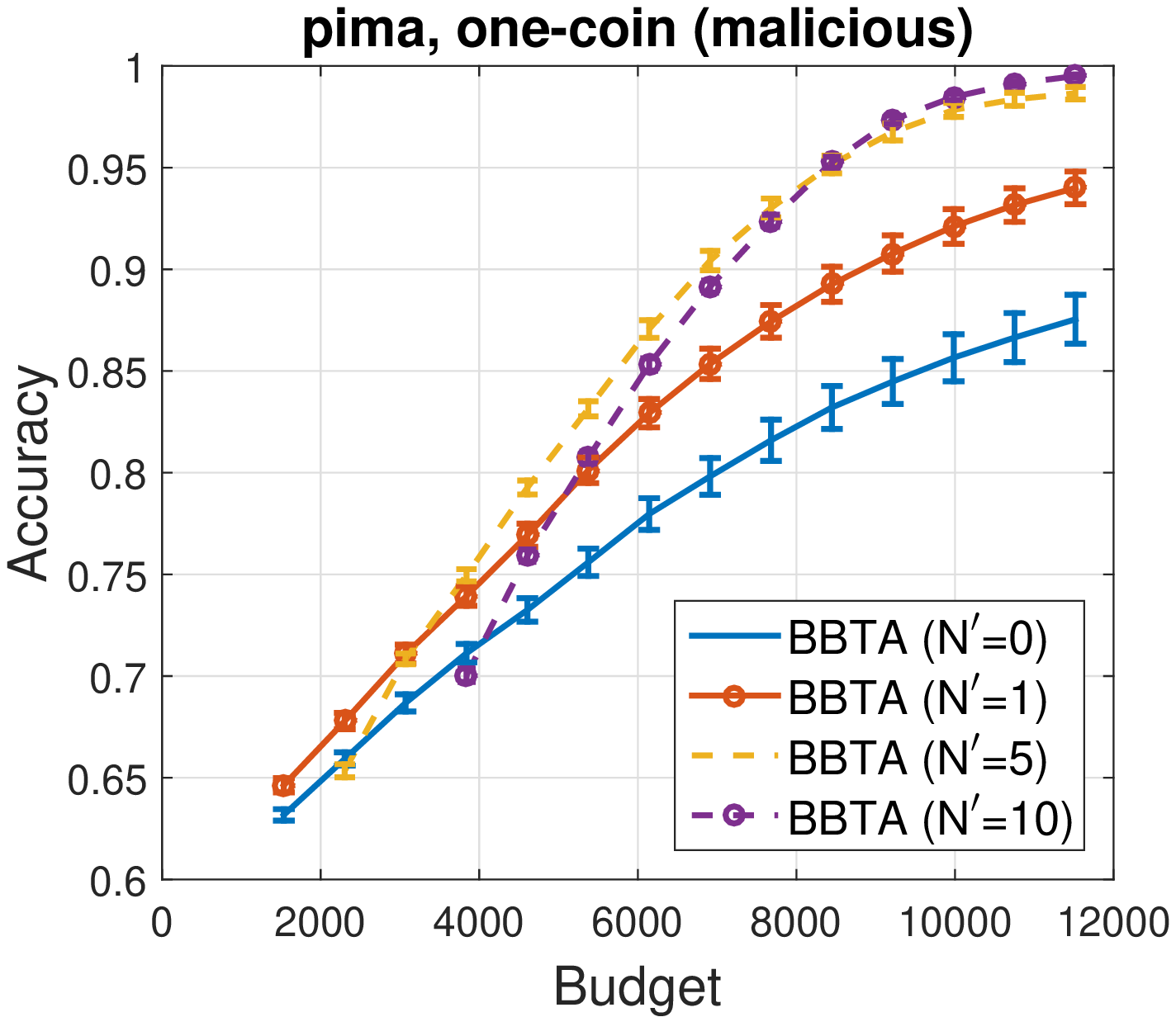}
\label{fig:pima-onecoin_m-nprime}
}
\caption{Results of changing $N'$ on three benchmark datasets with three worker models.}
\label{fig:bench-nprime}
\end{figure*}

We perform experiments on three popular UCI benchmark datasets\footnote{http://archive.ics.uci.edu/ml/}: ionosphere ($N=351$), breast ($N=569$), and pima ($N=768$). We consider instances in these datasets as labeling tasks in crowdsourcing. True labels of all tasks in these datasets are available. To simulate various heterogeneous cases in the real world, we first use \emph{k-means} to cluster these three datasets into $S=3,4,5$ subsets respectively (corresponding to different contexts). Since there are no crowd workers in these datasets, we then simulate workers ($K=30,40,50$, respectively) by using the following worker models in heterogeneous setting:
\begin{description}
\item[Spammer-Hammer Model]\hfill\\
A hammer gives true labels, while a spammer gives random labels \citep{karger11}. We introduce this model into heterogeneous setting: each worker is a hammer on one subset of tasks but a spammer on others.
\item[One-Coin Model]\hfill\\
Each worker gives true labels with a given probability (i.e. accuracy). This model is widely used in many existing crowdsourcing literatures (e.g. \citealp{raykar10,chen13}) for simulating workers. We use this model in heterogeneous setting: each worker gives true labels with higher accuracy (we set it to 0.9) on one subset of tasks, but with lower accuracy (we set it to 0.6) on others. 
\item[One-Coin Model (Malicious)]\hfill\\
This model is based on the previous one, except that we add more malicious labels: each worker is good at one subset of tasks (accuracy: 0.9), malicious or bad at another one (accuracy: 0.3), and normal at the rest (accuracy: 0.6).
\end{description}

With the generated labels from simulated workers, we can calculate the true accuracy for each worker by checking the consistency with the true labels. Figure \ref{fig:bench-workers} illustrates the counts of simulated workers with the true accuracy falling in the associated interval (e.g., 0.65 represents that the true accuracy is between 60\% and 65\%). It is easy to see that the spammer-hammer model and the one-coin model (malicious) create more adversarial environments than the one-coin model.

We compare BBTA with three state-of-the-art task assignment methods: IEThresh, CrowdSense and OptKG in terms of accuracy. Accuracy is calculated as the proportion of correct estimates for true labels. Also, since the budget is limited, we expect a task assignment method can achieve its highest accuracy as fast as possible as the budget increases (high convergence speed). We set $N'=0\text{ and }1$ for BBTA to see the effectiveness of the pure exploration phase. We also implement a naive baseline for comparison: we randomly select a task-worker pair at each step, and use majority voting mechanism for label aggregation. Accuracy of all methods is compared at different levels of budget. For BBTA, OptKG, and the naive baseline, we set the maximum amount of budget at $T=15N$. Since the budget is not pre-fixed in IEThresh and CrowdSense, we carefully select the threshold parameters for them, which affect the consumed budgets. Additionally, we also try to introduce the state-of-the-art methods (designed for the homogeneous setting) into the heterogeneous setting. Specifically, we split the total budget and allocate a sub-budget to a context in proportion to the number of tasks with this context. In particular, for context $s$, we allocate the sub-budget $T\cdot|\text{tasks with context }s|/N$. Then we can run an instance of a homogeneous method for each context within the associated sub-budget. Since OptKG has the most similar problem setting to that of the proposed method, it is straightforward to run multiple instances of OptKG with a pre-fixed sub-budget for each context. However, for the two heuristic methods IEThresh and CrowdSense, it is difficult to figure out how to use them in this way, since the budget could not be pre-determined in their settings. 

Figure \ref{fig:bench-acc} shows the averages and standard errors of accuracy as functions of budgets for all methods in nine cases (i.e. three datasets with three worker models). As we can see, BBTA with $N^\prime=1$ works better than that with $N^\prime=0$, indicating that the pure exploration phase helps in improving the performance. It is also shown that BBTA ($N^\prime=1$) outperforms other methods in all six cases with the spammer-hammer model and the one-coin model (malicious). This demonstrates that BBTA can handle spamming or malicious labels better than others in more adversarial heterogeneous environments. For the three cases with the one-coin model where there are more reliable labels, almost all methods have good performances. Nevertheless, IEThresh performs poorly in all cases, even worse than the naive baseline. The reason is that IEThresh samples a subset of reliable workers at each step, by calculating upper confidence intervals of workers based on their labeling performances on previous tasks. In heterogeneous settings, however, a worker reliable at previous tasks may be poor at next ones. This makes IEThresh learn workers' reliability incorrectly, resulting in poor sampling of worker subsets. Although CrowdSense also adopts the mechanism of dynamically sampling worker subsets, its exploration-exploitation criterion gives it a chance of randomly selecting some workers who may be reliable at next tasks. For OptKG, not surprisingly, OptKG (Multi.) which is aware of contexts outperforms the original OptKG. The tendency of either of them implies that they may achieve the best accuracy as the budget goes to infinity, but the convergence speed is shown to be slower than those of BBTA and CrowdSense. In crowdsourcing, it is important to achieve the best accuracy as fast as possible, especially when the budget is limited. 

We then run BBTA with changing $N'$, to see how $N'$ affects the performance. We set $N'= 0,1,5,10$, and the results are shown in Figure~\ref{fig:bench-nprime}. It can be seen that without the pure exploration phase (i.e. $N'=0$), the performance is the worst in all nine cases. On the other hand, when we add the pure exploration phase ($N'>0$), the performance is improved. However, we are unable to conclude that the larger $N'$ is, the better the performance is (e.g. $N'=10$ does not always make the method achieve its highest accuracy fastest in all nine cases). Indeed, a larger $N'$ means a longer pure exploration phase, which consumes a larger proportion of the total budget. For example, when $N'=10$, the performance usually starts from a lower accuracy level than that when we choose other exploration lengths. Although its start level is lower, as the budget increases, the performance when $N'=10$ can outperform all the others in most cases of the spammer-hammer and one-coin (malicious) models. However, it can only achieve the same level as the performance when $N'=5$ in all cases of the one-coin model, but with a lower convergence speed. In BBTA, we can only choose $N'$ to affect the performance, and there are also some other factors such as the true reliability of workers and how different their labeling performances are from each others, of which we usually do not have prior knowledge in real-world crowdsourcing. If we could somehow know beforehand that the worker pool is complex (in terms of the difference of workers' reliability) as in the spammer-hammer and one-coin (malicious) models, setting a larger $N'$ may help, otherwise a practical choice would be to set a small $N'$.

\subsection{Real Data}

Next, we compare BBTA with the existing task assignment methods on two real-world datasets.

\subsubsection{Recognizing Textual Entailment}

We first use a real dataset from \emph{recognizing textual entailment} (RTE) tasks in natural language processing. This dataset is collected by \citet{snow08} using \emph{Amazon Mechanical Turk} (MTurk). For each RTE task in this dataset, the worker is presented with two sentences and given a binary choice of whether the second sentence can be inferred from the first one. The true labels of all tasks are available and used for evaluating the performances of all task assignment methods in our experiments.

In this dataset, there is no context information available, or we can consider all tasks have the same context. That is, this is a homogeneous dataset ($S=1$). The numbers of tasks and workers in this dataset are $N=800$ and $K=164$ respectively. Since the originally collected label set is not complete (i.e. not every worker gives a label for each task), we decided to use a matrix completion algorithm\footnote{We use \emph{GROUSE} \citep{balzano10} for label matrix completion in our experiments.} to fill the incomplete label matrix, to make sure that we can collect a label when any task is assigned to any worker in the experiments. Then we calculate the true accuracy of workers for this dataset, as illustrated in Figure~\ref{fig:rte-workers}.

\begin{figure*}[t]
\centering
\subfigure[$N=800,K=164,S=1$]{
\includegraphics[width=0.47\textwidth]{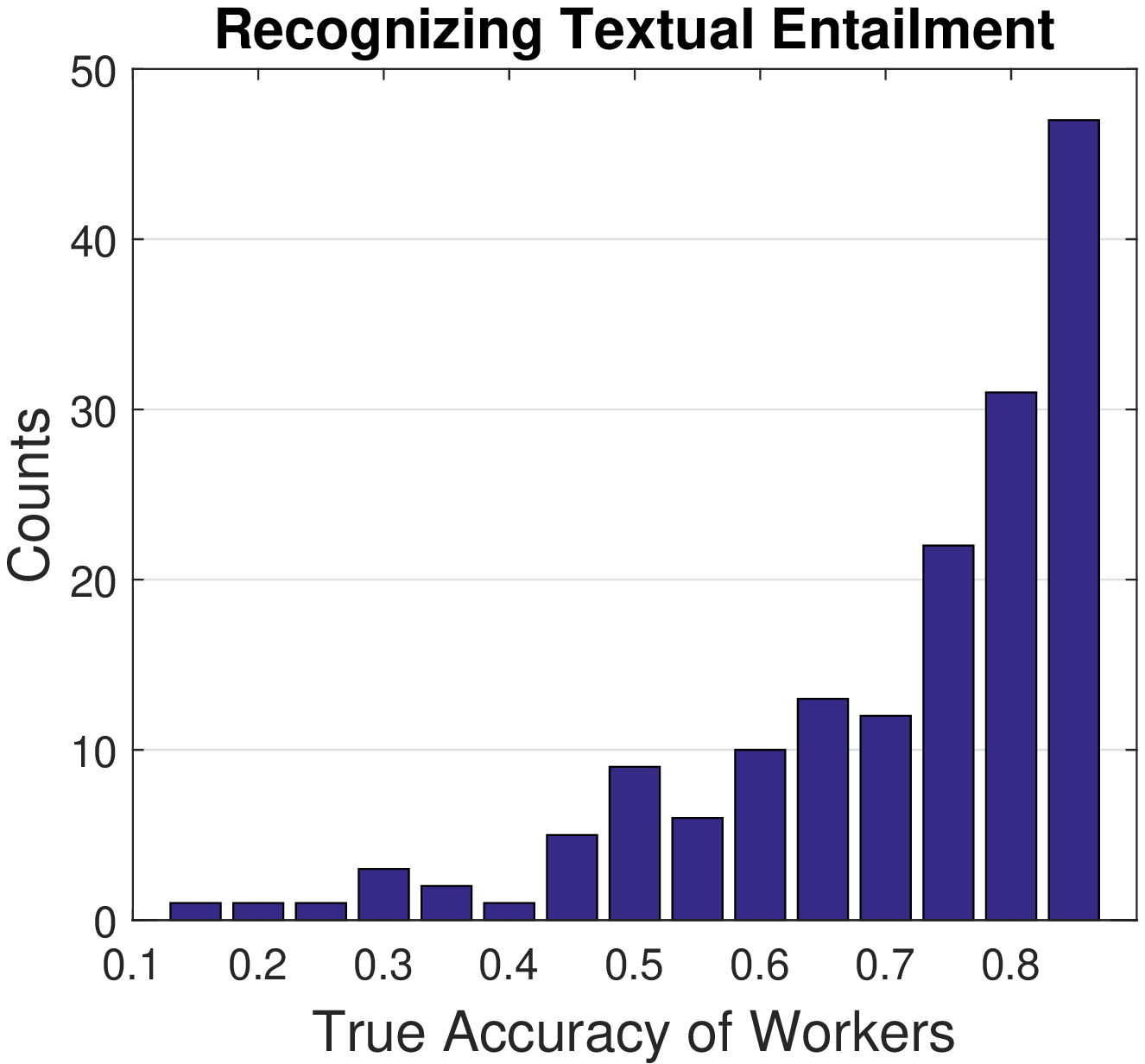}
\label{fig:rte-workers}
}
\subfigure[$N=204,K=42,S=2$]{
\includegraphics[width=0.47\textwidth]{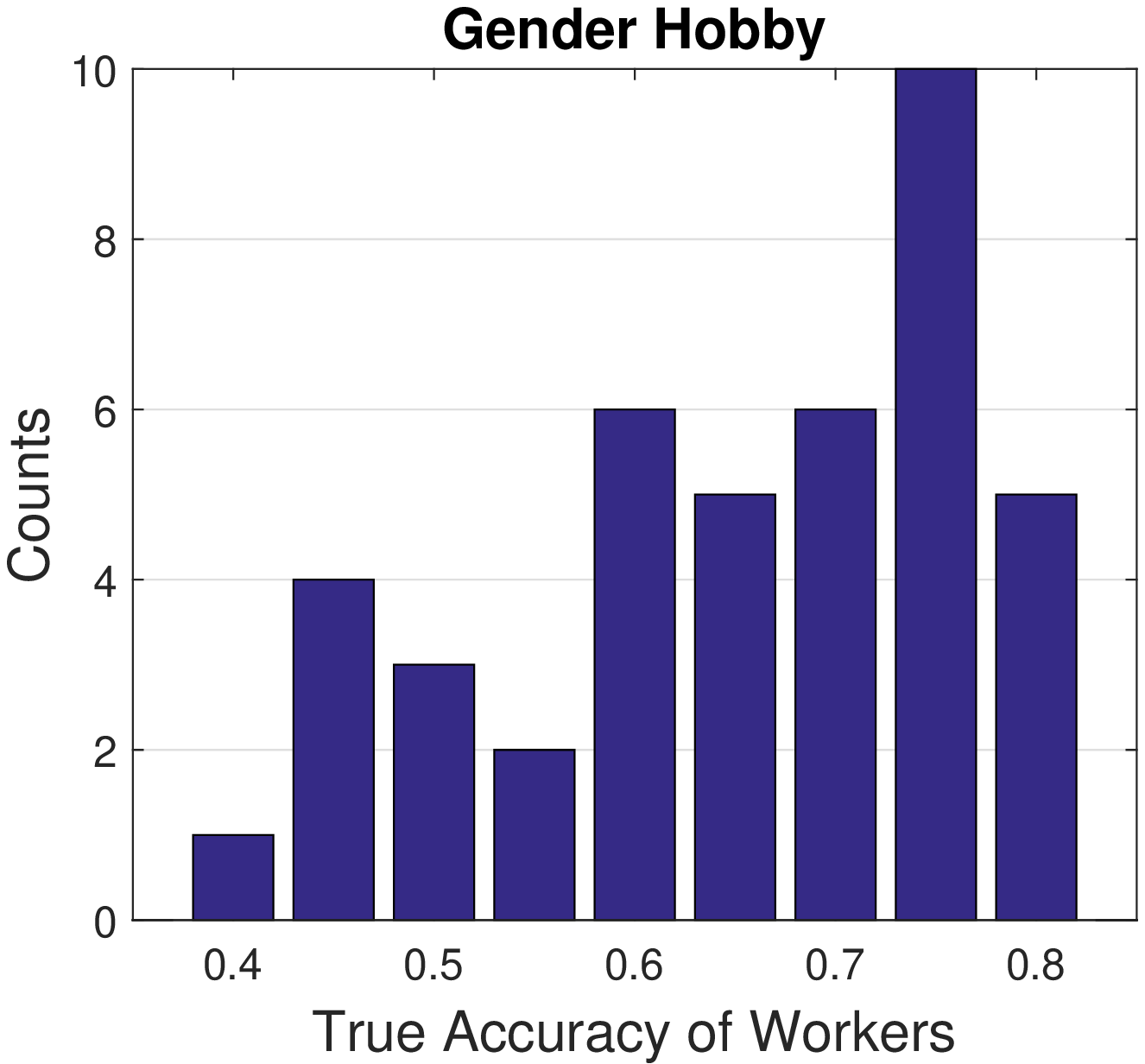}
\label{fig:gh-workers}
}
\caption{Distribution of true accuracy of workers for two real-world datasets.}
\label{fig:real-workers}
\end{figure*}

\begin{figure*}[t]
\centering
\subfigure[$N=800,K=164,S=1$]{
\includegraphics[width=0.47\textwidth]{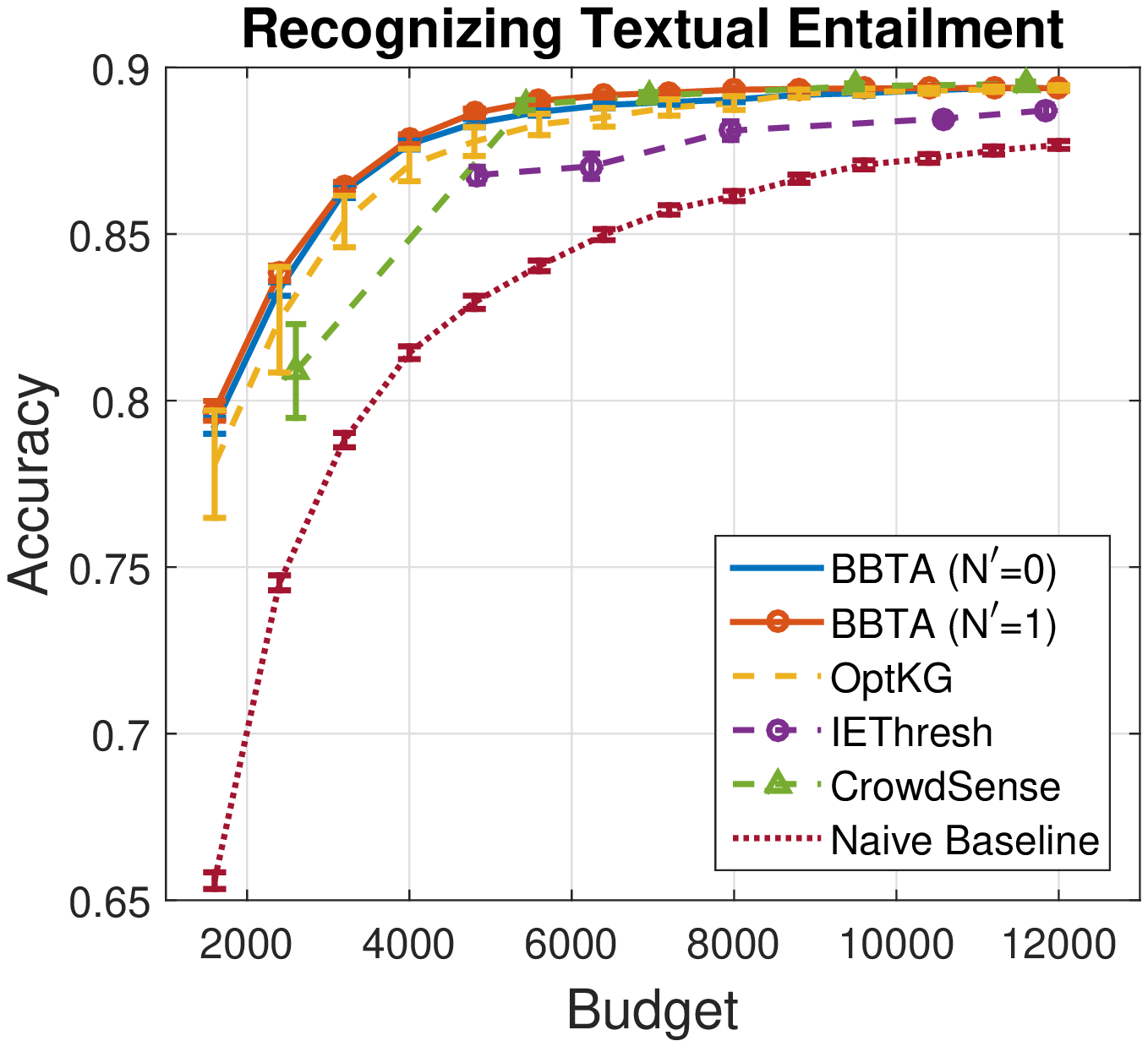}
\label{fig:rte-acc}
}
\subfigure[$N=204,K=42,S=2$]{
\includegraphics[width=0.47\textwidth]{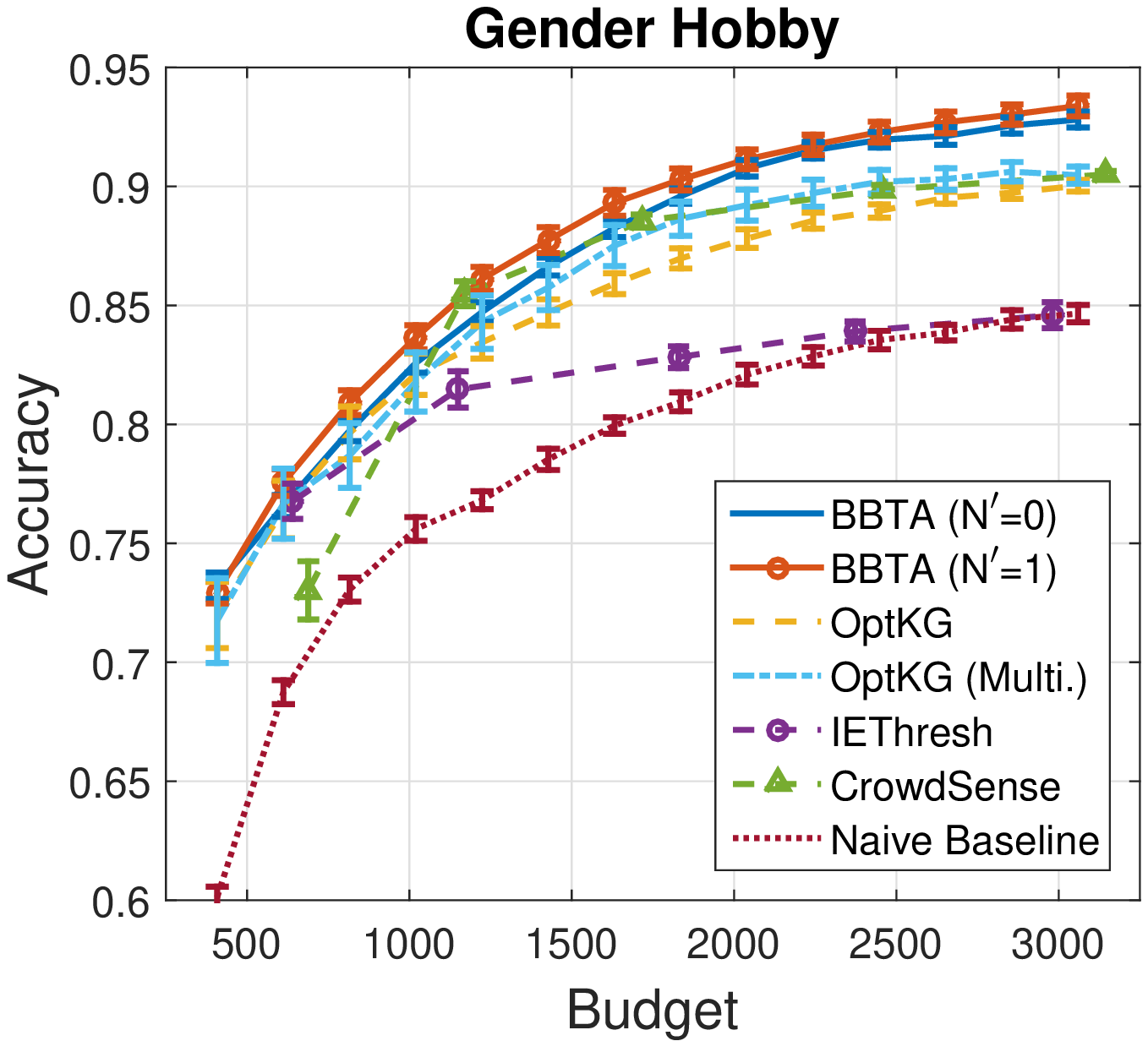}
\label{fig:gh-acc}
}
\caption{Comparison results on two real-world datasets.}
\label{fig:real-acc}
\end{figure*}

Figure~\ref{fig:rte-acc} depicts the comparison results on the RTE data, showing that all methods work very well. The reason is that there is a significant proportion of reliable workers in this dataset as we can see in Figure~\ref{fig:rte-workers}, and finding them out is not a difficult mission for all methods. It is also shown in Figure~\ref{fig:rte-acc} that BBTA with $N'=1$ converges to the highest accuracy slightly faster than others. This is important in practice especially in the budget-sensitive setting, because achieving higher accuracy within a lower budget is always favored from the perspective of requesters in crowdsourcing.

\subsubsection{Gender Hobby Dataset}

The second real dataset we use is \emph{Gender Hobby} (GH) collected from MTurk by \citet{mo13}. Tasks in this dataset are binary questions that are explicitly divided into two contexts ($S=2$): sports and makeup-cooking. This is a typical heterogeneous dataset, where there are $N=204$ tasks ($102$ per context) and $K=42$ workers. Since the label matrix in the original GH data is also incomplete, we use the matrix completion algorithm again to fill the missing entries. Figure~\ref{fig:gh-workers} illustrates the distribution of the true accuracy of workers in this dataset. It is easy to see that the labels given by the workers in this dataset are more malicious than those in the RTE data (Figure~\ref{fig:rte-workers}) due to the increased diversity of tasks.

Figure~\ref{fig:gh-acc} plots the experimental results, showing that BBTA with $N'=0$ and $N'=1$ outperform others on this typical heterogeneous dataset.

\section{Conclusions and Future Work}
\label{sec:conclusion}

In this paper, we proposed a contextual bandit formulation to address the problem of task assignment in heterogeneous crowdsourcing. In the proposed method called the bandit-based task assignment (BBTA), we first explored workers' reliability and then attempted to adaptively assign tasks to appropriate workers.

We used the exponential weighting scheme to handle the exploration-exploitation trade-off in worker selection and utilized the weighted voting mechanism to aggregate the collected labels. Thanks to the contextual formulation, BBTA models the task-dependent reliability for workers, and thus is able to intelligently match workers to tasks they can do best. This is a significant advantage over the state-of-the-art task assignment methods. We experimentally showed the usability of BBTA in heterogeneous crowdsourcing tasks.

We also theoretically investigated the regret bounds for BBTA. In the regret analysis, we showed the performance of our strategy converges to that of the optimal one as the budget goes to infinity.

Heterogeneity is practical and important in recent real-world crowdsoucing systems. There is still a lot of room for further work in heterogeneous setting. In particular, we consider four possible directions:
\begin{itemize}
\item In practical crowdsourcing, categorical labeling tasks with more than two classes are also common besides binary ones. Then we can extend the current method to the categorical setting, where the weighted voting mechanism turns to 
\begin{align*}
\widehat{y}_i=\argmax_{c\in\cC}\sum_{j=1}^{K}w^s_{j}\cdot\dOne_{y_{i,j}=c},\text{ where } \cC \text{ is the categorical label space.}
\end{align*}
\item The similarities among tasks with different types (corresponding to the information sharing among different contexts) can also be considered, since real-world workers may have similar behaviors on different types of tasks.
\item To further improve the reliability of the crowdsourcing system, we can involve the supervision of domain experts, and adaptively collect labels from both workers and experts. In such a scenario, we need to cope with the balance between the system reliability and the total budget, since expert labels (considered as the ground truth) are much more expensive than worker labels.
\item From a theoretical point of view, the context and loss could be considered to be dependent on the history of actions, although reasonably modeling this dependence relation is challenging in the setting of heterogeneous crowdsourcing. The current method considered the sequence of contexts and losses as external feedback, and thus adopted a standard bandit formulation. If we could somehow appropriately capture the dependence relation mentioned above, it would be possible to further improve the current theoretical results.
\end{itemize}

In our future work, we will further investigate the challenging problems in heterogeneous crowdsourcing.

\section*{Acknowledgments}
HZ and YM were supported by the MEXT scholarship and the CREST program. MS was supported by MEXT KAKENHI 25700022.

\appendix
\section*{Appendix}

\section{Proof of Lemma \ref{lm}}
This proof is following the framework of regret analysis for adversarial bandits by \citet{bubeck10}.

Recall that the unbiased estimator of loss is
\begin{align*}
\widetilde{l}_{j,t}=\frac{l_{j,t}}{p_{j,t}}\dOne_{j=j_t}.
\end{align*}
Then we have
\begin{align*}
&\bE_{j\sim p_t}\left[\widetilde{l}_{j,t}\right]=\sum_{j=1}^{K}\left(p_{j,t}\frac{l_{j,t}}{p_{j,t}}\dOne_{j=j_t}\right)=l_{j_t,t},\\
&\bE_{j_t\sim p_t}\left[\widetilde{l}_{j',t}\right]=\sum_{j=1}^{K}\left(p_{j,t}\frac{l_{j',t}}{p_{j',t}}\dOne_{j'=j}\right)=l_{j',t},
\end{align*}
which imply
\begin{align}
\label{eq:diff-loss}
\sum_{t=1}^{T_2}l_{j_t,t}-\sum_{t=1}^{T_2}l_{j',t}=\sum_{t=1}^{T_2}\bE_{j\sim p_t}\left[\widetilde{l}_{j,t}\right]-\sum_{t=1}^{T_2}\bE_{j_t\sim p_t}\left[\widetilde{l}_{j',t}\right].
\end{align}

Now we rewrite $\bE_{j\sim p_t}\left[\widetilde{l}_{j,t}\right]$ in the first term of \eqref{eq:diff-loss} as 
\begin{align}
\label{eq:rewrite}
\bE_{j\sim p_t}\left[\widetilde{l}_{j,t}\right]=\frac{1}{\eta_t}\ln\bE_{j\sim p_t}\left[\exp\left(-\eta_t\left(\widetilde{l}_{j,t}-\bE_{j'\sim p_t}\left[\widetilde{l}_{j',t}\right]\right)\right)\right]-\frac{1}{\eta_t}\ln\bE_{j\sim p_t}\left[\exp\left(-\eta_t\widetilde{l}_{j,t}\right)\right].
\end{align}
Then the first term of \eqref{eq:rewrite} is
\begin{align}
\label{eq:1st-term-rewrite}
\frac{1}{\eta_t}\ln\bE_{j\sim p_t}\left[\exp\left(-\eta_t\left(\widetilde{l}_{j,t}-\bE_{j'\sim p_t}\left[\widetilde{l}_{j',t}\right]\right)\right)\right]
&=\frac{1}{\eta_t}\ln\bE_{j\sim p_t}\left[\exp\left(-\eta_t\widetilde{l}_{j,t}\right)\right]+\bE_{j'\sim p_t}\left[\widetilde{l}_{j',t}\right]\notag\\
&\leq\frac{1}{\eta_t}\bE_{j\sim p_t}\left[\exp\left(-\eta_t\widetilde{l}_{j,t}\right)-1\right]+l_{j_t,t}\notag\\
&=\frac{1}{\eta_t}\bE_{j\sim p_t}\left[\exp\left(-\eta_t\widetilde{l}_{j,t}\right)-1+\eta_t\widetilde{l}_{j,t}\right]\notag\\
&\leq\frac{1}{\eta_t}\bE_{j\sim p_t}\left[\frac{\eta_t^2\widetilde{l}_{j,t}^2}{2}\right]\notag\\
&\leq\frac{\eta_t}{2p_{j_t,t}},
\end{align}
where the first and second inequalities are due to $\ln x\leq x-1$ for $x>0$ and $\exp(x)\leq 1+x+x^2/2$ for $x\leq 0$ respectively, and the third inequality is because
\begin{align*}
\bE_{j\sim p_t}\left[\widetilde{l}^2_{j,t}\right]=\sum_{j=1}^{K}p_{j,t}\left(\frac{l_{j,t}}{p_{j,t}}\dOne_{j=j_t}\right)^2=\frac{l^2_{j_t,t}}{p_{j_t,t}}\leq\frac{1}{p_{j_t,t}}.
\end{align*}
We define the potential function as
\begin{align*}
\Phi_t(\eta)=\frac{1}{\eta}\ln\frac{1}{K}\sum_{j=1}^{K}\exp\left(-\eta L_{j,t}\right).
\end{align*}
Then the second term of \eqref{eq:rewrite} is
\begin{align}
\label{eq:2nd-term-rewrite}
-\frac{1}{\eta_t}\ln\bE_{j\sim p_t}\left[\exp\left(-\eta_t\widetilde{l}_{j,t}\right)\right]&=-\frac{1}{\eta_t}\ln\sum_{j=1}^{K} p_{j,t}\exp\left(-\eta_t\widetilde{l}_{j,t}\right)\notag\\
&=-\frac{1}{\eta_t}\ln\sum_{j=1}^{K}\frac{w_{j,t}}{\sum_{j'=1}^{K}w_{j',t}}\exp\left(-\eta_t\widetilde{l}_{j,t}\right)\notag\\
&=-\frac{1}{\eta_t}\ln\frac{\sum_{j=1}^{K}\exp\left(-\eta_t L_{j,t-1}-\eta_t\widetilde{l}_{j,t}\right)}{\sum_{j'=1}^{K}\exp\left(-\eta_t L_{j',t-1}\right)}\notag\\
&=-\frac{1}{\eta_t}\ln\frac{\sum_{j=1}^{K}\exp\left(-\eta_t L_{j,t}\right)}{\sum_{j'=1}^{K}\exp\left(-\eta_t L_{j',t-1}\right)}\notag\\
&=\Phi_{t-1}(\eta_t)-\Phi_t(\eta_t).
\end{align}

Combining \eqref{eq:diff-loss}, \eqref{eq:rewrite}, \eqref{eq:1st-term-rewrite}, and \eqref{eq:2nd-term-rewrite}, we obtain

\begin{align}
\label{eq:diff-loss-bound}
\sum_{t=1}^{T_2}l_{j_t,t}-\sum_{t=1}^{T_2}l_{j',t}\leq\sum_{t=1}^{T_2}\frac{\eta_t}{2p_{j_t,t}}+\sum_{t=1}^{T_2}\left(\Phi_{t-1}(\eta_t)-\Phi_t(\eta_t)\right)-\sum_{t=1}^{T_2}\bE_{j_t\sim p_t}\left[\widetilde{l}_{j',t}\right].
\end{align}
Since we have
\begin{align*}
\bE_{j_t\sim p_t}\left[\frac{1}{p_{j_t,t}}\right]=\sum_{j=1}^{K}p_{j,t}\frac{1}{p_{j,t}}=K,
\end{align*}
then the expectation of the first term of \eqref{eq:diff-loss-bound} is
\begin{align}
\label{eq:1st-term-diff-loss-bound}
\bE\left[\sum_{t=1}^{T_2}\frac{\eta_t}{2p_{j_t,t}}\right]=\frac{1}{2}\bE\left[\sum_{t=1}^{T_2}\bE_{j_t\sim p_t}\left[\frac{\eta_t}{p_{j_t,t}}\right]\right]=\frac{K}{2}\sum_{t=1}^{T_2}\eta_t.
\end{align}
The second term of \eqref{eq:diff-loss-bound} is transformed into
\begin{align}
\label{eq:2nd-term-diff-loss-bound}
\sum_{t=1}^{T_2}\left(\Phi_{t-1}(\eta_t)-\Phi_t(\eta_t)\right)=\Phi_0(\eta_1)+\sum_{t=1}^{T_2-1}\left(\Phi_{t}(\eta_{t+1})-\Phi_t(\eta_t)\right)-\Phi_{T_2}(\eta_{T_2}).
\end{align}
By the definition of the potential function, we have
\begin{align}
\label{eq:phi0}
\Phi_0(\eta_1)&=\frac{1}{\eta_1}\ln\frac{1}{K}\sum_{j=1}^{K}\exp\left(-\eta_1 L_{j,0}\right)\notag\\
&\leq\frac{1}{\eta_1}\left(\frac{1}{K}\sum_{j=1}^{K}\exp\left(-\eta_1L_{j,0}\right)-1\right)\notag\\
&\leq\frac{1}{\eta_1K}\left(\sum_{j=1}^{K}\left(\exp\left(-\eta_1L_{j,0}\right)-1+\eta_1L_{j,0}\right)\right)\notag\\
&\leq\frac{1}{\eta_1K}\sum_{j=1}^{K}\frac{\eta_1^2L_{j,0}^2}{2}\notag\\
&=\frac{\eta_1}{2K}\sum_{j=1}^{K}\left(\sum_{i\in\cI_1}\dOne_{y_{i,j}\neq \widehat{y}_i}\right)^2\notag\\
&\leq\frac{\eta_1}{2K}\sum_{j=1}^{K}\left(N'\sum_{i\in\cI_1}\left(\dOne_{y_{i,j}\neq \widehat{y}_i}\right)^2\right)\notag\\
&=\frac{\eta_1N'}{2K}\sum_{i\in\cI_1}\sum_{j=1}^{K}\left(\dOne_{y_{i,j}\neq \widehat{y}_i}\right)^2\notag\\
&\leq\frac{\eta_1N'}{2K}\sum_{i\in\cI_1}\frac{K}{2}=\frac{\eta_1{N'}^2}{4},
\end{align}
where the first and third inequalities are due to $\ln x\leq x-1$ for $x>0$ and $\exp(x)\leq 1+x+x^2/2$ for $x\leq 0$ respectively, the fifth line of \eqref{eq:phi0} is obtained by the definition of the cumulative loss $L_{j,0}$ in the pure exploration phase, the fourth inequality of \eqref{eq:phi0} is an immediate result of using Jensen's inequality and $\left|\cI_1\right|=N'$, and the last line of \eqref{eq:phi0} is because we use majority voting in the pure exploration phase. We also have
\begin{align}
\label{eq:phiT}
-\Phi_{T_2}(\eta_{T_2})&=-\frac{1}{\eta_{T_2}}\ln\frac{1}{K}\sum_{j=1}^{K}\exp\left(-\eta_{T_2} L_{j,T_2}\right)\notag\\
&\leq \frac{\ln K}{\eta_{T_2}}-\frac{1}{\eta_{T_2}}\ln\left(\exp\left(-\eta_{T_2}L_{j',T_2}\right)\right)\notag\\
&=\frac{\ln K}{\eta_{T_2}}+L_{j',0}+\sum_{t=1}^{T_2}\widetilde{l}_{j',t}\notag\\
&\leq\frac{\ln K}{\eta_{T_2}}+N'+\sum_{t=1}^{T_2}\widetilde{l}_{j',t}.
\end{align}

Now we combine \eqref{eq:diff-loss-bound}, \eqref{eq:1st-term-diff-loss-bound}, \eqref{eq:2nd-term-diff-loss-bound}, \eqref{eq:phi0}, and \eqref{eq:phiT} and obtain
\begin{align*}
\bE\left[\sum_{t=1}^{T_2}l_{j_t,t}-\sum_{t=1}^{T_2}l_{j',t}\right]\leq\frac{K}{2}\sum_{t=1}^{T_2}\eta_t+\frac{\eta_1{N'}^2}{4}+\frac{\ln K}{\eta_{T_2}}+N'+\bE\left[\sum_{t=1}^{T_2-1}\left(\Phi_{t}(\eta_{t+1})-\Phi_t(\eta_t)\right)\right].
\end{align*}
Since it is shown by \citet{bubeck10} that $\Phi'(\eta)\geq0$ and we set $\eta_{t}=\sqrt{\frac{\ln K}{tK}}$, we have
\begin{align}
\label{eq:bound}
\bE\left[\sum_{t=1}^{T_2}l_{j_t,t}-\sum_{t=1}^{T_2}l_{j',t}\right] &\leq\frac{K}{2}\sum_{t=1}^{T_2}\sqrt{\frac{\ln K}{tK}}+\frac{{N'}^2}{4}\sqrt{\frac{\ln K}{K}}+\sqrt{T_2K\ln K}+N'\notag\\
&\leq 2\sqrt{(T-KN')K\ln K}+\frac{{N'}^2}{4}\sqrt{\frac{\ln K}{K}}+N',
\end{align}
where the second inequality is because $T_2=T-KN'$ and
\begin{align*}
\sum_{t=1}^{T_2}\frac{1}{\sqrt{t}}\leq 2\sqrt{T_2}.
\end{align*}
\eqref{eq:bound} shows the bound about our strategy competing against \emph{any} strategy that always selects one single worker, thus concluding the proof.
\bibliographystyle{apa}
\bibliography{fullbib}

\end{document}